\documentclass{article}
\PassOptionsToPackage{sort,round,compress}{natbib}
\usepackage[preprint]{neurips_2026} %

\usepackage[utf8]{inputenc} %
\usepackage[T1]{fontenc}    %
\usepackage[hypertexnames=false]{hyperref} %
\usepackage{url}            %
\usepackage{booktabs}       %
\usepackage{amsfonts}       %
\usepackage{nicefrac}       %
\usepackage{microtype}      %
\usepackage{xcolor}         %

\usepackage{graphicx}
\usepackage{subfigure}
\usepackage{multicol,multirow,makecell}
\usepackage{adjustbox}   
\usepackage{threeparttable}

\usepackage{amsmath}
\usepackage{amssymb}
\usepackage{mathtools}
\usepackage{amsthm}
\usepackage{dsfont,bm,bbm}
\usepackage{cases}

\usepackage[capitalize,noabbrev]{cleveref}

\usepackage[textsize=tiny]{todonotes}

\def\eqref#1{(\ref{#1})}

\def\floor#1{\lfloor #1 \rfloor}
\def\1{\bm{1}}

\def\vzero{{\bm{0}}}

\def\vtheta{{\bm{\theta}}}
\def\va{{\bm{a}}}
\def\vb{{\bm{b}}}

\def\ve{{\bm{e}}}

\def\vv{{\bm{v}}}

\def\vx{{\bm{x}}}
\def\vy{{\bm{y}}}
\def\vz{{\bm{z}}}

\def\mA{{\bm{A}}}
\def\mB{{\bm{B}}}
\def\mC{{\bm{C}}}
\def\mD{{\bm{D}}}
\def\mE{{\bm{E}}}

\def\mG{{\bm{G}}}
\def\mH{{\bm{H}}}
\def\mI{{\bm{I}}}

\def\mM{{\bm{M}}}

\def\mQ{{\bm{Q}}}

\def\mV{{\bm{V}}}

\def\mX{{\bm{X}}}

\def\mZ{{\bm{Z}}}

\DeclareMathAlphabet{\mathsfit}{\encodingdefault}{\sfdefault}{m}{sl}
\SetMathAlphabet{\mathsfit}{bold}{\encodingdefault}{\sfdefault}{bx}{n}

\def\gA{{\mathcal{A}}}
\def\gB{{\mathcal{B}}}

\def\gD{{\mathcal{D}}}
\def\gE{{\mathcal{E}}}
\def\gF{{\mathcal{F}}}
\def\gG{{\mathcal{G}}}

\def\gK{{\mathcal{K}}}
\def\gL{{\mathcal{L}}}

\def\gO{{\mathcal{O}}}
\def\gP{{\mathcal{P}}}

\def\gX{{\mathcal{X}}}

\def\gZ{{\mathcal{Z}}}

\def\sN{{\mathbb{N}}}

\def\sP{{\mathbb{P}}}

\def\sR{{\mathbb{R}}}

\newcommand{\E}{\mathbb{E}}

\newcommand{\KL}{D_{\mathrm{KL}}}

\DeclareMathOperator*{\argmax}{arg\,max}
\DeclareMathOperator*{\argmin}{arg\,min}

\DeclareMathOperator{\tr}{tr}

\usepackage[algo2e, linesnumbered, ruled, vlined]{algorithm2e}

\usepackage{doi}
\usepackage[many]{tcolorbox}
\usepackage{cancel}
\usepackage{pifont}
\usepackage[normalem]{ulem}
\usepackage{soul}
\usepackage{subcaption}
\usepackage{enumitem}
\usepackage[table]{xcolor}

\definecolor{lightgrey}{rgb}{0.9, 0.9, 0.9}

\tcbset{
  baseboxstyle/.style={
    boxrule=0pt,          
    frame hidden,          
    sharp corners,         
    before skip=10pt,      
    after skip=10pt,       
    left=2pt, right=2pt,   
    top=2pt, bottom=2pt    
  }
}

\tcolorboxenvironment{definition}{baseboxstyle, colback=black!5}   
\tcolorboxenvironment{theorem}{baseboxstyle, colback=blue!5}      
\tcolorboxenvironment{corollary}{baseboxstyle, colback=blue!5}   
\tcolorboxenvironment{lemma}{baseboxstyle, colback=teal!5} 
\tcolorboxenvironment{proposition}{baseboxstyle, colback=blue!5}

\definecolor{darkblue}{rgb}{0.0,0.0,0.65}
\definecolor{darkred}{rgb}{0.65,0.0,0.0}
\definecolor{darkgreen}{rgb}{0.0,0.5,0.0}
\definecolor{tab:blue}{RGB}{31,119,180}  
\definecolor{tab:red}{RGB}{214,39,40}  
\definecolor{tab:green}{RGB}{44,160,44}  
\definecolor{tab:orange}{RGB}{255,127,14}  

\hypersetup{
    colorlinks = true,
    citecolor  = darkblue,
    linkcolor  = darkred,
    filecolor  = darkblue,
    urlcolor   = darkgreen,
}

\usepackage{aliascnt}

\newtheorem{theorem}{Theorem}[section]

\newaliascnt{lemma}{theorem}
\newtheorem{lemma}[lemma]{Lemma}
\aliascntresetthe{lemma}

\newaliascnt{proposition}{theorem}
\newtheorem{proposition}[proposition]{Proposition}
\aliascntresetthe{proposition}

\newaliascnt{corollary}{theorem}
\newtheorem{corollary}[corollary]{Corollary}
\aliascntresetthe{corollary}

\newaliascnt{assumption}{theorem}
\newtheorem{assumption}[assumption]{Assumption}
\aliascntresetthe{assumption}

\newaliascnt{definition}{theorem}
\newtheorem{definition}[definition]{Definition}
\aliascntresetthe{definition}

\newaliascnt{remark}{theorem}
\newtheorem{remark}[remark]{Remark}
\aliascntresetthe{remark}

\theoremstyle{plain}

\newtheorem{oracle}{Oracle}

\allowdisplaybreaks

\newcommand{\ANReg}{\mathrm{AN\text{-}Reg}}
\newcommand{\MNReg}{\mathrm{MN\text{-}Reg}}
\newcommand{\ABRReg}{\mathrm{ABR\text{-}Reg}}
\newcommand{\MBRReg}{\mathrm{MBR\text{-}Reg}}
\newcommand{\DGap}{\mathrm{DGap}}
\newcommand{\Reg}{\mathrm{Reg}}

\newcommand{\dmu}{\dot{\mu}}
\newcommand{\ddmu}{\ddot{\mu}}

\newcommand{\nuc}{\mathrm{nuc}}

\def\bignorm#1{\left\lVert #1 \right\rVert}
\def\bignormop#1{\left\lVert #1 \right\rVert_{\mathrm{op}}}
\def\bignormnuc#1{\left\lVert #1 \right\rVert_{\mathrm{nuc}}}

\newcommand{\indicator}{\mathds{1}}

\newcommand{\Skew}{\mathrm{Skew}}

\renewcommand{\vec}{\mathrm{vec}}
\newcommand{\poly}{\mathrm{poly}}

\newcommand{\Ber}{\mathrm{Ber}}
\newcommand{\rank}{\mathrm{rank}}
\newcommand{\SNE}{\textbf{\textit{SNE}}}
\newcommand{\GBPM}{\textbf{\textit{GBPM}}}

\crefname{theorem}{Theorem}{Theorems}
\Crefname{theorem}{Theorem}{Theorems}

\crefname{lemma}{Lemma}{Lemmas}
\Crefname{lemma}{Lemma}{Lemmas}

\crefname{proposition}{Proposition}{Propositions}
\Crefname{proposition}{Proposition}{Propositions}

\crefname{corollary}{Corollary}{Corollaries}
\Crefname{corollary}{Corollary}{Corollaries}

\crefname{assumption}{Assumption}{Assumptions}
\Crefname{assumption}{Assumption}{Assumptions}

\crefname{definition}{Definition}{Definitions}
\Crefname{definition}{Definition}{Definitions}

\crefname{remark}{Remark}{Remarks}
\Crefname{remark}{Remark}{Remarks}

\usepackage{autonum} %
\title{Provably Efficient Regularized Online RLHF with Generalized Bilinear Preferences}

\author{%
Junghyun Lee \\
KAIST AI \\
\texttt{jh\_lee00@kaist.ac.kr}
\And
Minju Hong \\
KAIST EE \\
\texttt{minju23@kaist.ac.kr}
\And
Kwang-Sung Jun \\
POSTECH CSE/AI \\
\texttt{kwangsungjun@postech.ac.kr}
\And
Chulhee Yun \\
KAIST AI \\
\texttt{chulhee.yun@kaist.ac.kr}
\And
Se-Young Yun \\
KAIST AI \\
\texttt{yunseyoung@kaist.ac.kr}
}

\begin{document}

\maketitle

\begin{abstract}
    We consider the problem of \emph{regularized} best-response max-regret minimization in online RLHF under general preferences and bandit feedback.
    While various regularizers are utilized to robustify alignment, known polylogarithmic regret guarantees remain heavily specific to KL.
    To investigate whether such fast rates extend beyond KL, we adopt the \emph{Generalized Bilinear Preference Model (GBPM)}—capturing intransitive preferences over $d$-dimensional item-wise features via a rank-$2r$ skew-symmetric matrix—to isolate the impact of generic regularization.
    Crucially, under GBPM, we prove that the dual gap of any greedy policy is bounded by the \emph{squared} estimation error, derived using \emph{only} strong convexity and skew-symmetry.
    Under a feature coverage assumption, we establish a \emph{generic} polylogarithmic regret of $\tilde{\gO}(\eta d^4 C_{\min}^{-1} (\log T)^2 \wedge d^2 C_{\min}^{-1/2} \sqrt{T})$ with Greedy Sampling, and a dimension-wise improved regret (for well-conditioned arm-sets) of $\tilde{\gO}(C_{\min}^{-2} \sqrt{\eta r T} \wedge r^{1/3} C_{\min}^{-4/3} T^{2/3})$ with Explore-Then-Commit, where $\eta^{-1}$ is the regularization coefficient, $T$ is the time horizon, and $C_{\min}$ is an arm-set dependent quantity.
    This demonstrates that ``fast'' regrets are \emph{not} KL-specific, but rather a fundamental consequence of generic strongly convex geometry.
\end{abstract}

\section{Introduction}
\label{sec:introduction}

\paragraph{General Preference Learning.}
Aligning large language models (LLMs) with human values has emerged as a central challenge in modern AI~\citep{Llama,Qwen2.5,GPT4}.
While the common approach to Reinforcement Learning from Human Feedback (RLHF) heavily relies on reward-based models like the Bradley-Terry-Luce (BTL) model~\citep{bradleyterry1952,christiano2017rlhf}, scalar utilities inherently struggle to capture the cyclic, intransitive, and diverse nature of human preferences~\citep{may1954intransitivity,tversky1969intransitivity}. 
This representational bottleneck has motivated a shift toward \textit{General Preference Learning} (or \textit{Nash Learning}), which directly targets the Nash equilibrium (NE) of a preference game~\citep{nash1951,vonNeumann1928,mckelvey-palfrey}.
Recently, this game-theoretic perspective has demonstrated notable empirical promise in LLM alignment~\citep{munos2024nash,ye2024general,cui2024ultrafeedback,rosset2024nash}.

In both practice and theory, solving these preference games online under bandit feedback relies heavily on optimizing \emph{regularized} objectives. In practical RLHF, purely maximizing an unregularized, observed preference often leads to reward hacking, diversity collapse, and hallucinatory text generation~\citep{michaud2020,tien2023causal,casper2023survey}. Embedding regularization directly into the objective prevents large drifts from reference models (e.g., SFT models or expert demonstrations) and robustifies the resulting equilibrium—a concept dating back to the seminal work on \emph{quantal response equilibria} by \citet{mckelvey-palfrey}.

Consequently, understanding the statistical efficiency of these regularized games, often measured via \textbf{\textit{regularized best-response regret}}, has become a primary target in modern RL, game-theoretic alignment, and learning in regularized zero-sum game literature~\citep{munos2024nash,wu2025greedy,xiong2024iterative,ye2024general,nayak2025logarithmic,yang2025incentivize}.
In particular, our analysis focuses on \textbf{\textit{regularized max-regret}}, which measures suboptimality of the \textit{max} player only (see \Cref{sec:regret-def} for its definition).
This notion originates from self-play frameworks for RL in two-player zero-sum games~\citep{bai2020self-play,bai2020self-play2,liu2021self-play,jin2022exploiter,xiong2022self-play}, and has become the standard metric in theoretical analyses of online RLHF under general preferences~\citep{ye2024general,wu2025greedy}. The intuition is that the learner ultimately only cares about obtaining the NE policy for the max-player, the policy actually deployed in practice.

\paragraph{The KL-Centric Theories.}
The current theoretical landscape for establishing polylogarithmic regularized best-response regret remains overwhelmingly KL-centric.
Recent theoretical advances in reward-based RLHF rely heavily on the explicit closed-form Gibbs densities unique to KL-regularized bandits~\citep{xiong2024iterative,zhao2025logarithmic,zhao2025sharp,wu2025greedy,ji2026kl}.
Consequently, the corresponding literature on regularized best-response regret has followed this same KL-specific trajectory~\citep{wu2025greedy,nayak2025logarithmic,yang2025incentivize}.

On the other hand, much RL and LLM literature has increasingly explored strongly convex penalties \emph{other than} reverse KL to achieve distinct structural benefits.
For example, negative Shannon entropy is used in max-entropy exploration~\citep{ziebart2008entropy,neu2017entropy,haarnoja2018soft}; mixtures of reverse KL divergences are used to yield more diverse and less biased alignment~\citep{le2025multiple,aminian2025multiple}; the chi-squared divergence is known to provably mitigate overoptimization~\citep{huang2025chi-squared}; and Tsallis entropy can encourage sparse policy selection~\citep{lee2018tsallis,chow2018tsallis}.
Broader classes of regularizers -- such as $\alpha$-R\'enyi entropies~\citep{zhang2026renyi}, Csisz\'ar $f$-divergences~\citep{wang2024f-divergence,go2023f-divergence,xu2025f-divergence}, and expected or strongly convex regularizers~\citep{yang2019regularized,geist2019regularized} -- have been studied as well.

Because prior max-regret analyses rely entirely on the \emph{closed-form solution} of the KL-regularized bandit~\citep{ye2024general,wu2025greedy,nayak2025logarithmic}, it is unclear whether similar polylogarithmic regret can be attained for regularized games with non-KL regularizers.
This motivates our central theoretical question:
\begin{center}
    \emph{Is the fast regularized max-regret achievable under \underline{any} strongly convex regularizer beyond KL?}
\end{center}

\paragraph{The GBPM Abstraction and Theoretical Setup.}
To answer this, we require a theoretically tractable mathematical abstraction, analogous to the linear BTL model~\citep{bradleyterry1952,plackett1975}. To this end, we adopt the \textbf{\textit{Generalized Bilinear Preference Model (GBPM)}}~\citep{lee2025gl-lowpopart,zhang2025bilinear}: given \emph{item-wise} features $\bm\phi^1, \bm\phi^2 \in \sR^d$, the preference probability is modeled as $P^*(\bm\phi^1 \succ \bm\phi^2) := \mu\left( (\bm\phi^1)^\top \bm\Theta_\star \bm\phi^2 \right),$ where $\mu(\cdot)$ is a link function satisfying $\mu(z) + \mu(-z) = 1$, and $\bm\Theta_\star \in \sR^{d \times d}$ is a \emph{skew-symmetric} matrix of rank at most $2r < d$.
By considering such a contextual counterpart to the linear BTL model~\citep{zhu2023rlhf}, GBPM allows us to isolate the statistical complexity of online RLHF and rigorously analyze the impact of general regularizers.

\paragraph{Contributions.} 
Under the GBPM framework, we demonstrate that ``fast'' regret rates are achievable for \textbf{\textit{any}} strongly convex regularizer.
Our technical contributions are two-fold:
\begin{itemize}[leftmargin=*]\setlength{\itemsep}{0pt}
    \item \textbf{Quadratic Bound on Dual Gap.} We show that the dual gap of any greedy NE policy is upper bounded by the \textit{square} of the estimation error of $\bm\Theta_\star$ for \textbf{\textit{any}} strongly convex regularizer. Our analysis leverages the skew-symmetry of GBPM, the strong convexity of the regularized game objective, and the integral probability metric representation of the $\ell_1$-distance~\citep{muller1997ipm} to derive a novel, \emph{self-bounding quadratic inequality} \textbf{(\cref{sec:regularized})}.
    
    \item \textbf{Fast Max-Regrets.} We establish fast regret bounds using two different algorithms tailored to distinct regimes, both critically utilizing the novel quadratic error bound.
    These results are shown under a feature coverage assumption, which introduces an arm-set dependent quantity $C_{\min} > 0$ (\textbf{\Cref{sec:coverage}}).
    First, we prove that \texttt{Greedy Sampling (GS)} achieves \emph{polylogarithmic} regrets of $\widetilde{\gO}( \eta d^4 C_{\min}^{-1} (\log T)^2 \wedge d^2 C_{\min}^{-1/2} \sqrt{T})$ \textbf{(\cref{sec:regret-logarithmic})}.
    Second, to address the high-dimensional regime, we demonstrate that \texttt{Explore-Then-Commit (ETC)} with nuclear-norm regularized MLE achieves regrets of $\widetilde{\gO}\left( C_{\min}^{-2} \sqrt{\eta r T}\right)$ and $\tilde{\gO}(r^{1/3} C_{\min}^{-4/3} T^{2/3})$ \textbf{(\cref{sec:regret-high-dim})}.
\end{itemize}

\section{Problem Setting}
\label{sec:problem-setting}

We study contextual preference learning with contexts $\vx \in \gX$ and actions $\va \in \gA$. A policy is a mapping $\pi : \gX \to \Delta(\gA)$, where $\pi(\cdot \mid \vx)$ denotes the conditional distribution over actions given $\vx$. We denote the policy class by $\Pi$. Given two actions $\va^1,\va^2$ and a context $\vx$, the event $\va^1 \succ \va^2 \mid \vx$ means that response $\va^1$ is preferred to response $\va^2$ under context $\vx$.

\subsection{Generalized Bilinear Preference Model (GBPM)}
\label{sec:gbpm}
We first introduce the low-rank contextual general preference model that we consider in this work.
Define $\Skew(d; 2r, S)$ as
$\left\{ \bm\Theta \in \Skew(d) : \rank(\bm\Theta) \leq 2r, \bignormnuc{\bm\Theta} \leq S \right\}$,
where $\Skew(d) := \left\{ \bm\Theta \in \sR^{d \times d} : \bm\Theta^\top = -\bm\Theta \right\}.$ We assume a known feature map $\phi:\gX\times\gA\to\gB^d(1) \triangleq \{ \bm\phi \in \sR^d : \|\bm\phi\|_2 \leq 1 \}$.
Now the definition of \GBPM~\citep{zhang2025bilinear,lee2025gl-lowpopart}:

\begin{definition}[\textbf{\textit{Generalized Bilinear Preference Model}}]
\label{def:gbpm}
    Let rank $r \leq \floor{d/2}$ and norm bound $S > 0$ be known. For any $\vx \in \gX, \va^1, \va^2 \in \gA$, the ground-truth preference under \GBPM~is:
    \begin{equation}
        P^*(\va^1 \succ \va^2 \mid \vx) := \mu\left( \phi(\vx, \va^1)^\top \bm\Theta_\star \phi(\vx, \va^2) \right),
    \end{equation}
    where $\bm\Theta_\star \in \Skew(d; 2r, S)$ is an unknown \textbf{skew-symmetric, low-rank} matrix, and $\mu : \mathbb{R} \rightarrow [0, 1]$ is a known link satisfying: for some $L_\mu \geq \kappa > 0$,
    \begin{enumerate}
        \item $\mu$ is twice differentiable, monotone increasing, and symmetric ($\mu(z) + \mu(-z) = 1$).
        \item $\kappa \leq \dmu\left( \phi^\top \bm\Theta \phi' \right) \leq L_\mu, \ \left| \ddmu\left( \phi^\top \bm\Theta \phi' \right) \right| \leq L_\mu, \ \ \forall \bm\phi, \bm\phi' \in \gB^d(1), \bm\Theta \in \Skew(d; 2r, S)$.\footnote{We assume the same constant $L_\mu$ for the upper bounds of $\dmu$ and $|\ddmu|$ for simplicity of the exposition.}
    \end{enumerate}
\end{definition}

The conditions for $\mu$ are standard in logistic and generalized linear (dueling) bandits~\citep{faury2020logistic,abeille2021logistic,lee2024logistic,lee2024glm,lee2025gl-lowpopart,wu2024dueling,bengs2022dueling}.
The logistic link $\mu(z)=(1+e^{-z})^{-1}$ satisfies the above with $L_\mu=\frac{1}{4}$.
The linear link $\mu(z)=\tfrac12+z$ is also covered, in which case $L_\mu=1$~\citep{gajane2015dueling,wu2024dueling}.
Lastly, one can see that $P^*$ is anti-symmetric: $P^*(\va^1 \succ \va^2 \mid \vx) + P^*(\va^2 \succ \va^1 \mid \vx) = 1$.

\begin{remark}
    While recent works have introduced contextual bandit frameworks for general preference learning~\citep{yang2025incentivize,nayak2025logarithmic,wu2024dueling}, they predominantly rely on \emph{item pair-wise} feature maps (linearizing payoffs as $\langle \varphi(\va^1, \va^2), \vtheta \rangle$ for each pair of actions $\va^1, \va^2 \in \gA$) or tabular structures~\citep{odonoghue2021matrix}; the former is conceptually similar to \citet{wu2024dueling}, though the connection is not explicitly drawn in the literature.
    This contrasts with practical RLHF scenarios in which only \emph{item-wise} features $\phi(\va)$ are available, motivating our choice to consider \GBPM{}; see \citet{zhang2025bilinear} and \citet[Appendix K]{lee2025gl-lowpopart} for more detailed discussions.
\end{remark}

\subsection{Population Regularized Game and Nash Equilibrium}
\label{sec:game}

We first extend the action-level preference to evaluate policies. The ground-truth preference of $\pi^1$ (\emph{\color{blue}max-player}) over $\pi^2$ (\emph{\color{red}min-player}) given a context $\vx$ is defined by marginalizing over their action distributions, defined as 
$P^*(\pi^1 \succ \pi^2 \mid \vx) := \E_{{\va^1 \sim \pi^1(\cdot \mid \vx) , \, \va^2 \sim \pi^2(\cdot \mid \vx)}} \left[ P^*(\va^1 \succ \va^2 \mid \vx) \right]$.

To evaluate these policies globally, we take the expectation over the unknown context distribution $d_0 \in \Delta(\gX)$. For any parameter $\bm\Theta \in \Skew(d)$, we define the expected population game objective as:
\begin{equation}
    J(\pi^1, \pi^2; \bm\Theta) := \E_{\vx \sim d_0} \E_{\va^i \sim \pi^i(\cdot \mid \vx)} \left[ \mu\left(\phi(\vx, \va^1)^\top \bm\Theta \phi(\vx, \va^2)\right) \right].
\end{equation}

The ground-truth population preference is then precisely $\E_{\vx \sim d_0} [P^*(\pi^1 \succ \pi^2 \mid \vx)] = J(\pi^1, \pi^2; \bm\Theta_\star)$. For notational convenience, we abbreviate this true objective as $J(\pi^1, \pi^2)$, and when analyzing realized features, we utilize the shorthand $J(\bm\phi^1, \bm\phi^2; \bm\Theta) := \mu\left((\bm\phi^1)^\top \bm\Theta \bm\phi^2\right)$.
We now introduce its regularized counterpart.
For $\eta \in (0, \infty]$ and a $\beta^{-1}$-strongly convex regularizer $\psi : \Delta(\gA) \rightarrow \sR_{\geq 0}$ w.r.t. $\bignorm{\cdot}_1$, we define a \emph{symmetric}, regularized game objective $J_{\eta} : \Pi \times \Pi \rightarrow \sR$ as follows:
\begin{equation}
    J_{\eta}(\pi, \pi'; \bm\Theta) := J(\pi, \pi'; \bm\Theta) - \eta^{-1} \E_{\vx \sim d_0}[\psi(\pi(\cdot \mid \vx))] + \eta^{-1} \E_{\vx \sim d_0}[\psi(\pi'(\cdot \mid \vx))].
\end{equation}

Standard solution concepts (e.g., Condorcet winners) may not exist in general preference learning~\citep{dudik2015dueling,bengs2021survey,munos2024nash,swamy2024minimaximalist}.
Thus, as in many recent literature in online RLHF~\citep{munos2024nash}, we consider \emph{(regularized) Nash Equilibrium (NE)}~\citep{nash1951,mckelvey-palfrey}:
\begin{definition}
    A pair $(\pi^1_\star, \pi^2_\star) \in \Pi \times \Pi$ is a \textbf{Nash equilibrium (NE)} if for all $\pi^1, \pi^2 \in \Pi$:
    \begin{equation}
        J_{\eta}(\pi^1, \pi^2_\star) \leq J_{\eta}(\pi^1_\star, \pi^2_\star) \leq J_{\eta}(\pi^1_\star, \pi^2).
    \end{equation}
    If $\pi^1_\star = \pi^2_\star =: \pi_\star$, we refer to $\pi_\star$ as a \textbf{symmetric NE (SNE)}.
    By the minimax theorem~\citep{vonNeumann1928,sion1958minimax}, any \textbf{SNE} $\pi_\star$ is equivalently characterized as follows:
    \begin{equation}
        \pi^\star \in \argmax_{\pi^1 \in \Pi} \min_{\pi^2 \in \Pi} J_{\eta}(\pi^1, \pi^2).
    \end{equation}
\end{definition}

\subsection{Online Interaction Protocol and Regularized Max-Regret}
\label{sec:regret-def}

The online contextual RLHF protocol proceeds as follows: At each $t = 1, \dots, T$, a context $\vx_t \sim d_0$ is revealed. The learner chooses policies ${\color{blue}\hat{\pi}_t^1(\cdot | \vx_t)}$ and ${\color{red}\hat{\pi}_t^2(\cdot | \vx_t)}$, samples actions ${\color{blue}\va_t^1}$ and ${\color{red}\va_t^2}$, and receives a bandit feedback $r_t \sim \Ber(P^*({\color{blue}\va_t^1} \succ {\color{red}\va_t^2} \mid \vx_t))$.
This constitutes a contextual symmetric two-player zero-sum game~\citep{balduzzi2019open} with bandit feedback.
Note that this is a \emph{self-play} framework, as the learner controls \emph{both} players to learn by playing against itself to compute a NE.

We evaluate any resulting policy sequence $\{ ({\color{blue}\hat{\pi}_t^1}, {\color{red}\hat{\pi}_t^2}) \}_{t \in [T]}$ using \textit{\textbf{regularized max-regret}}:
\begin{equation}
    \MBRReg_{\eta}(T) := \sum_{t=1}^T \DGap_{\eta}({\color{blue}\hat{\pi}_t^1}), \quad
    \DGap_{\eta}({\color{blue}\hat{\pi}_t^1}) := \frac{1}{2} - \min_{\pi^2 \in \Pi} J_{\eta}({\color{blue}\hat{\pi}_t^1}, \pi^2),
\end{equation}
where we refer to $\DGap_{\eta}({\color{blue}\hat{\pi}_t^1})$ as the \textbf{(symmetric) dual gap} of a policy ${\color{blue}\hat{\pi}_t^1} \in \Pi$, which quantifies how close $\DGap_{\eta}({\color{blue}\hat{\pi}_t^1})$ is to an \SNE.

\begin{remark}
    While there are multiple definitions of ``regret'' in a two-player zero-sum game, our work focuses on max-regret (specifically, the Max-Best-Response regret).
    As discussed in \Cref{sec:introduction}, this has been widely considered in self-play RL in two-player zero-sum games~\citep{bai2020self-play,bai2020self-play2,liu2021self-play,jin2022exploiter,xiong2022self-play} and theoretical analyses of online RLHF under general preferences~\citep{ye2024general,wu2025greedy}.
    Notably, the max-regret can be converted to a sample 
    for finding an NE via online-to-batch conversion~\citep{freund-schapire}.
    We refer the reader to Appendix~\ref{app:other-regrets} for a detailed discussion of alternative regret definitions.
\end{remark}

\section{A New Analysis of Regularized Regret}
\label{sec:regularized}

We first present our main technical contribution: a novel bound on the instantaneous dual gap of any greedy NE policy.
Denoting $\bm\phi \sim \pi$ as sampling a $\bm\phi(\vx, \va) \in \gB^d(1)$ from $\va \sim \pi(\cdot \mid \vx)$ and $\vx \sim d_0$:

\begin{theorem}
\label{thm:regularized}
    For \underline{any} estimator $\widehat{\bm\Theta}_t \in \Skew(d)$ at time $t$, define the max-player's policy as
    \begin{equation}
    \label{eqn:NE}
        \hat{\pi}_t \gets \argmax_{\pi^1} \min_{\pi^2} J_{\eta}(\pi^1, \pi^2; \widehat{\bm\Theta}_t).
    \end{equation}
    Then, the instantaneous dual gap is bounded as follows: denoting $\mE_t := \bm\Theta_\star - \widehat{\bm\Theta}_t$,
    \begin{equation}
        \DGap_{\eta}({\color{blue}\hat{\pi}_t}) \leq L_\mu \min\left\{ (2 L_\mu \eta \beta + 1) \E_{\bm\phi \sim {\color{blue}\hat{\pi}_t}}\left[ \bignorm{\mE_t \bm\phi}_2^2 \right], \ \sqrt{\E_{\bm\phi}\left[ \bignorm{\mE \bm\phi}_2^2 \right]} + \frac{1}{2} \E_{\bm\phi}\left[ \bignorm{\mE \bm\phi}_2^2 \right] \right\}.
    \end{equation}
\end{theorem}

Crucially, this bound holds for \underline{\emph{any}} choice of estimator and \underline{\emph{any}} $\beta^{-1}$-strongly convex regularizer $\psi(\cdot)$. As long as $\eta < \infty$ (i.e., the regularization by $\psi$ exists), the instantaneous dual gap is bounded \emph{quadratically} by the expected estimation error of $\bm\Theta_\star$ \emph{along the features} of ${\color{blue}\hat{\pi}_t}$. As we will see in the proof, without strong convexity, one only recovers a linear dependence $\E_{\bm\phi}[\bignorm{\mE_t \bm\phi}_2]$.

\subsection{\texorpdfstring{Proof of \cref{thm:regularized}}{Proof of Theorem 3.1}}

\paragraph{Regret Decomposition and Taylor Expansion.}
For simplicity let us omit dependencies on $t$.
Let $\tilde{\pi} = \argmin_{\pi \in \Pi} J_{\eta}(\hat{\pi}, \pi)$ be the min-player's best response to $\hat{\pi}$ with respect to the true objective.
We denote the \emph{\color{brown}dual gap} of $\hat{\pi}$ as ${\color{brown}X} := \frac{1}{2} - J_{\eta}(\hat{\pi}, \tilde{\pi})$.
Decomposing the regret, we have that
\begin{align}
    {\color{brown}X} = J_{\eta}(\hat{\pi}, \tilde{\pi}; \widehat{\bm\Theta}) - J_{\eta}(\hat{\pi}, \tilde{\pi}) + \frac{1}{2} - J_{\eta}(\hat{\pi}, \tilde{\pi}; \widehat{\bm\Theta})
    &\overset{(*)}{\leq} J_{\eta}(\hat{\pi}, \tilde{\pi}; \widehat{\bm\Theta}) - J_{\eta}(\hat{\pi}, \tilde{\pi}) \\
    &= J(\hat{\pi}, \tilde{\pi}; \widehat{\bm\Theta}) - J(\hat{\pi}, \tilde{\pi}), \tag{regularization terms cancel out}
\end{align}
where the inequality $(*)$ holds due to the following reasoning.
First, we establish the following important property of symmetric game:
\begin{lemma}
\label{lem:symmetric}
    For any $\widehat{\bm\Theta} \in \Skew(d)$, the value of the game, $\max_{\pi^1} \min_{\pi^2} J_{\eta}(\pi^1, \pi^2; \widehat{\bm\Theta})$, is $\frac{1}{2}$.
\end{lemma}
\begin{proof}
    For any $\widehat{\bm\Theta} \in \Skew(d)$, there always exists a \SNE~\citep[Lemma 2.1]{swamy2024minimaximalist}.
    As the game value of \SNE~is $\frac{1}{2}$, it must be so for any \textbf{NE}~\citep{vonNeumann1928,sion1958minimax}.
\end{proof}
With this, we have that $\frac{1}{2} = \min_{\pi \in \Pi} J_{\eta}(\hat{\pi}, \pi; \widehat{\bm\Theta})$ for any given $\widehat{\bm\Theta} \in \Skew(d)$.
Then, it is easy to see that $\min_{\pi \in \Pi} J_{\eta}(\hat{\pi}, \pi; \widehat{\bm\Theta}) \leq J_{\eta}(\hat{\pi}, \tilde{\pi}; \widehat{\bm\Theta})$.

Let us denote $\E := \E_{\vx \sim d_0}\E_{\bm\phi \sim \hat{\pi}(\cdot \mid \vx), \tilde{\bm\phi} \sim \tilde{\pi}(\cdot \mid \vx)}$ when clear from the context.
Inspired by the regret analyses of logistic and generalized linear bandits~\citep{abeille2021logistic,lee2024glm,lee2024logistic}, applying a \textbf{\emph{\color{teal}Taylor expansion with integral remainder}} yields:
\begin{align}
    &J(\hat{\pi}, \tilde{\pi}; \widehat{\bm\Theta}) - J(\hat{\pi}, \tilde{\pi}) \\
    & = \underbrace{- \E\left[ \dmu(\bm\phi^\top \bm\Theta_\star \tilde{\bm\phi}) \bm\phi^\top \mE \tilde{\bm\phi} \right]}_{(a)} 
    + \underbrace{\E\left[ \left[ \int_0^1 (1 - z) \ddmu\left( \bm\phi^\top \left( \bm\Theta_\star - z \mE \right) \tilde{\bm\phi} \right) dz \right] (\bm\phi^\top \mE \tilde{\bm\phi})^2 \right]}_{(b)}.
\end{align}

\paragraph{Bounding the First-Order Term $(a)$.}
There are two key technical lemmas that are crucial in obtaining the self-bounding inequality later.

The first lemma relates $(a)$ to ${\color{magenta}D} \coloneq \sqrt{{\color{magenta}\E_{\vx \sim d_0}\left[ \bignorm{\hat{\pi}(\cdot \mid \vx) - \tilde{\pi}(\cdot \mid \vx)}_1^2 \right]}}$.
Its proof, which combines skew-symmetry and the variational representation of the $\ell_1$-norm~\citep{muller1997ipm}, is presented at the end of this subsection:
\begin{lemma}
\label{lem:anti-symmetry}
    $\left| \E\left[ \dmu(\bm\phi^\top \bm\Theta_\star \tilde{\bm\phi}) \bm\phi^\top \mE \tilde{\bm\phi} \right] \right| \leq L_\mu {\color{magenta}D} \sqrt{ \E_{\bm\phi}\left[ \bignorm{\mE \bm\phi}_2^2 \right]}.$
\end{lemma}

The second lemma, whose proof is deferred to \Cref{app:convexity}, allows for us to bound ${\color{magenta}D}$ with the dual gap ${\color{brown}X}$, other than the na\"{i}ve bound of ${\color{magenta}D} \leq 1$:
\begin{lemma}
\label{lem:strongly-convex}
    ${\color{brown}X} = \frac{1}{2} - \E_{\vx \sim d_0}[J_{\eta}(\hat{\pi}, \tilde{\pi} \mid \vx)] \ge (2 \eta \beta)^{-1} {\color{magenta}D}.$
\end{lemma}
We then chain everything, along with the na\"{i}ve bound of ${\color{magenta}D} \leq 1$, to obtain the following:
\begin{equation}
    (a) \leq L_\mu \sqrt{\left(1 \wedge 2 \eta \beta {\color{brown}X} \right) \E_{\bm\phi}\left[ \bignorm{\mE \bm\phi}_2^2 \right]}.
\end{equation}

\paragraph{Bounding the Second-Order Term $(b)$.}
We first bound $(b)$ by noting that $\ddmu(\cdot) \leq L_\mu$, giving $(b) \leq L_\mu \E[(\bm\phi^\top \mE \tilde{\bm\phi})^2] \int_0^1 (1 - z) dz = \frac{L_\mu}{2} \E[(\bm\phi^\top \mE \tilde{\bm\phi})^2].$
For clarity, let us distinguish between the independent expectations $\E_{\bm\phi}$ and $\E_{\tilde{\bm\phi}}$. We have:
\begin{align}
    \E[(\bm\phi^\top \mE \tilde{\bm\phi})^2] &= \E_{\bm\phi} \E_{\tilde{\bm\phi}}\left[ \tilde{\bm\phi}^\top \mE^\top \bm\phi \bm\phi^\top \mE \tilde{\bm\phi} \right] \\
    &\leq \E_{\bm\phi}\left[ \max_{\tilde{\bm\phi} \in \gB^d(1)} \tilde{\bm\phi}^\top (\mE^\top \bm\phi \bm\phi^\top \mE) \tilde{\bm\phi} \right]
    = \E_{\bm\phi}\left[ \bignorm{\mE \bm\phi}_2^2 \right],
\end{align}
where the last equality follows from the variational definition of the operator norm.\footnote{$\bignormop{\mE} = \sup_{\vx, \vy \in \gB^d(1)} \vx^\top \mE \vy$ for any $\mE \in \sR^{d \times d}$.}

\paragraph{Combining Everything.}
Combining these bounds for $(a)$ and $(b)$, we establish a \textbf{\emph{\color{teal}self-bounding quadratic inequality}} in the dual gap ${\color{brown}X}$:
\begin{equation}
\label{eqn:self-bounding}
    {\color{brown}X} \leq L_\mu \sqrt{\left(1 \wedge 2 \eta \beta {\color{brown}X} \right) \E_{\bm\phi}\left[ \bignorm{\mE \bm\phi}_2^2 \right]} + \frac{L_\mu}{2} \E_{\bm\phi}\left[ \bignorm{\mE \bm\phi}_2^2 \right].
\end{equation}
Solving for ${\color{brown}X}$ \emph{simultaneously} yields ${\color{brown}X} \leq 2 L_\mu^2 \eta \beta \E_{\bm\phi}[\bignorm{\mE \bm\phi}_2]^2 + L_\mu \E_{\bm\phi}[\bignorm{\mE \bm\phi}_2^2]$ \emph{and} ${\color{brown}X} \leq L_\mu \sqrt{\E_{\bm\phi}\left[ \bignorm{\mE \bm\phi}_2^2 \right]} + \frac{L_\mu}{2} \E_{\bm\phi}\left[ \bignorm{\mE \bm\phi}_2^2 \right].$
\qed

\begin{proof}[\underline{\textbf{Proof of Lemma~\ref{lem:anti-symmetry}}}]
    First, using the symmetry of $\mu$ and the skew-symmetry of $\mE$, we have:
    \begin{equation}
        Z \triangleq \E_{\vx \sim d_0}\E_{\bm\phi, \tilde{\bm\phi} \sim \hat{\pi}(\cdot \mid \vx)}\left[ \dmu(\bm\phi^\top \bm\Theta_\star \tilde{\bm\phi}) \bm\phi^\top \mE \tilde{\bm\phi} \right] = 0,
    \end{equation}
    where with a slight abuse of notation, $\bm\phi \sim \hat{\pi}(\cdot \mid \vx)$ denotes sampling $\bm\phi(\vx, \va)$ with $\va \sim \hat{\pi}(\cdot \mid \vx)$.
    Denote $f(\tilde{\bm\phi}; \bm\phi) := \dmu(\bm\phi^\top \bm\Theta_\star \tilde{\bm\phi}) \bm\phi^\top \mE \tilde{\bm\phi}$, which satisfies $\max_{\tilde{\bm\phi} \in \gB^d(1)} \left| f(\tilde{\bm\phi}; \bm\phi) \right| \leq L_\mu \bignorm{\mE \bm\phi}_2$.
    Then, 
    \begin{align}
        \left| \E\left[ \dmu(\bm\phi^\top \bm\Theta_\star \tilde{\bm\phi}) \bm\phi^\top \mE \tilde{\bm\phi} \right] \right| &= \left| \E_{\vx \sim d_0}\E_{\bm\phi \sim \hat{\pi}(\cdot \mid \vx), \tilde{\bm\phi} \sim \tilde{\pi}(\cdot \mid \vx)}\left[ f(\tilde{\bm\phi}; \bm\phi) \right] - Z \right| \\
        &\leq \E_{\vx \sim d_0}\E_{\bm\phi \sim \hat{\pi}(\cdot \mid \vx)}\left[ \left| \E_{\tilde{\bm\phi} \sim \hat{\pi}(\cdot \mid \vx)}\left[ f(\tilde{\bm\phi}; \bm\phi) \right] - \E_{\tilde{\bm\phi} \sim \tilde{\pi}(\cdot \mid \vx)}\left[ f(\tilde{\bm\phi}; \bm\phi) \right] \right| \right] \\
        &\overset{(*)}{\leq} \E_{\vx \sim d_0}\E_{\bm\phi \sim \hat{\pi}(\cdot \mid \vx)}\left[ L_\mu \bignorm{\mE \bm\phi}_2 \sup_{g \in \gG_\infty(1)} \left| \int g(\tilde{\bm\phi}) d(\hat{\pi}(\cdot \mid \vx) - \tilde{\pi}(\cdot \mid \vx))(\tilde{\bm\phi}) \right| \right] \\
        &\overset{(**)}{=} L_\mu \E_{\vx \sim d_0}\left[ \E_{\bm\phi \sim \hat{\pi}(\cdot \mid \vx)}\left[ \bignorm{\mE \bm\phi}_2 \right] \bignorm{\hat{\pi}(\cdot \mid \vx) - \tilde{\pi}(\cdot \mid \vx)}_1 \right],
    \end{align}
    where $(*)$ defines $\gG_\infty(1) := \left\{ g : \gB^d(1) \rightarrow [-1, 1] \mid \text{$g$ is measurable} \right\},$ and $(**)$ follows from the \emph{integral probability metric representation (IPM) of the $\ell_1$-norm}~\citep[Theorem 5.4]{muller1997ipm}.
    
    We conclude by decoupling the $\mE$ term and the $\ell_1$-error term via Cauchy-Schwarz as follows:
    \begin{align}
        \left| \E\left[ \dmu(\bm\phi^\top \bm\Theta_\star \tilde{\bm\phi}) \bm\phi^\top \mE \tilde{\bm\phi} \right] \right| &\leq L_\mu \sqrt{\E_{\vx \sim d_0}\left[ \E_{\bm\phi \sim \hat{\pi}(\cdot \mid \vx)}\left[ \bignorm{\mE \bm\phi}_2 \right]^2 \right]} \sqrt{{\color{magenta}\E_{\vx \sim d_0}\left[ \bignorm{\hat{\pi}(\cdot \mid \vx) - \tilde{\pi}(\cdot \mid \vx)}_1^2 \right]}} \\
        &\leq L_\mu {\color{magenta}D}\sqrt{\E_{\bm\phi}\left[ \bignorm{\mE \bm\phi}_2^2 \right]}. \tag{Jensen's inequality w.r.t. $\E_{\vx \sim d_0}[\cdot]$}
    \end{align}
\end{proof}

\subsection{Discussions}
Crucially, our proof relies strictly on the strong convexity of $\psi(\cdot)$ and entirely avoids any reliance on the specific algebraic properties of the KL divergence, departing from prior KL-centric analyses~\citep{wu2025greedy,nayak2025logarithmic,ye2024general}.
The central technical mechanism driving our proof is the \emph{\textbf{\color{teal}self-bounding inequality}} presented in Eqn.~\eqref{eqn:self-bounding}. 
While our use of a Taylor expansion is inspired by the regret analyses of logistic and generalized linear bandits~\citep{abeille2021logistic,lee2024glm}, our technical execution differs significantly. 
Prior works rely on self-concordance to control both terms and establish a similar self-bounding inequality; in contrast, our approach hinges on strong convexity, which is formalized through two key lemmas. 

First, \Cref{lem:anti-symmetry} ensures that the first-order term $(a)$ is upper-bounded by the distance ${\color{magenta}D}$ between our policy $\hat{\pi}$ and the adversary's worst-possible policy $\tilde{\pi}$.
To achieve this, the lemma leverages the skew-symmetry of the preference matrix and the symmetry of the link function.
Notably, the proof establishes a surprising connection to the integral probability metric (IPM) representation of the $\ell_1$-norm.
This connection, combined with an appropriate application of the Cauchy-Schwarz inequality, cleanly decouples the estimation error ${\color{magenta}D}$.

Second, \Cref{lem:strongly-convex} ensures that ${\color{magenta}D}$ is subsequently bounded by the dual gap.
By applying the zeroth-order characterization of strong convexity~\citep[Theorem 2.1.9]{nesterov}, we effortlessly relate the policy divergence to the suboptimality gap without needing to compute functional derivatives (since our optimality is defined with respect to policies).

\section{From Quadratic Error Bound to Fast Regrets}

\subsection{Feature Diversity and Computation Oracles}
\label{sec:coverage}

\paragraph{Feature Diversity.}
For the regret analyses, we consider the following assumption:

\begin{assumption}[Feature Diversity]
\label{ass:REC}
    The learner has access to an exploration policy $\rho(\cdot \mid \vx)$ such that $\lambda_{\min}\left( \E_{\vx \sim d_0} \E_{\va \sim \rho(\cdot \mid \vx)}\left[ \phi(\vx, \va) \phi(\vx, \va)^\top \right] \right) \geq C_{\min}$ for some $C_{\min} > 0$.
\end{assumption}

By considering this assumption, we can cleanly isolate the geometric impact of the strongly convex regularizer $\psi(\cdot)$ from the complexities of active exploration.
By abstracting away the exploration mechanism, we can directly answer our core theoretical question: whether fast rates are achievable under generic regularization when given sufficient data coverage.

We highlight four important aspects regarding the above assumption:
\begin{itemize}[leftmargin=*]
    \item \textbf{Theoretical Necessity.} In standard bandit settings, achieving fast rates typically requires explicit algorithmic exploration (e.g., optimism or Thompson Sampling).
    When relying on greedy sampling strategies (which we consider in \Cref{sec:regret-logarithmic}), assumptions regarding the diversity of the feature mapping are standard and necessary~\citep{goldenshluger-zeevi, kannan2018greedy, wu2020diverse, bastani2021greedy, bogunovic2021greedy, kim2024greedy}.
    Furthermore, in the high-dimensional regime (which we consider in \Cref{sec:regret-high-dim}), sufficient feature space coverage is information-theoretically unavoidable for obtaining dimension-wise improved regret bounds~\citep{hao2020sparse, li2022unified, zeng2025lowerbound}.

    \item \textbf{Empirical Plausibility.} While motivated by theoretical tractability, this assumption has reasonable analogues in practical RLHF pipelines \citep{dong2024rlhf, bai2022training}.
    Practitioners often use generative strategies that naturally induce diversity without explicit algorithmic exploration. 
    For example, querying an ensemble of LLMs or sampling from LLMs with various temperatures yield a diverse set of responses across the feature space \citep{troshin2025control, nguyen2025turning}.
    Under this interpretation, we argue that this assumption is a sensible theoretical simplification.
    
    \item \textbf{Statistical Tractability.} As detailed in Appendix~\ref{app:rho}, obtaining an approximate exploration policy (e.g., a $\rho$ satisfying the definition with $C_{\min} / 2$ or $C_{\min} / d$) can be achieved in a computationally tractable manner with a $T$-independent statistical cost in regret. This relies on the minimal requirement that the underlying context distribution $d_0(\cdot)$ has adequate coverage over the feature space; without it, the learner would face ``blind spots'' in $\sR^d$ and inherently fail to learn globally.
    
    \item \textbf{Scaling of $C_{\min}$.} The scaling of $C_{\min}$ is intrinsically tied to the geometry of the given feature map. Because $\|\phi(\vx, \va)\|_2 \leq 1$, it strictly follows that $C_{\min}^{-1} \geq d$. Well-conditioned sets (e.g., unit-normalized hypercubes, standard basis) yield $C_{\min}^{-1} \asymp d$, while ill-conditioned sets yield worse.
\end{itemize}

\paragraph{Computation Oracles.}
We lastly describe the computation model.
We assume the learner is \emph{tractable} (not necessarily efficient), accessing $\gA$ and $\Skew(d)$ only via:
\begin{oracle}
\label{oracle:sample}
\textbf{Sampling:} Given $\vx \in \gX$ and $\pi \in \Pi$, output a sample $\va \sim \pi(\cdot | \vx)$.
\end{oracle}
\begin{oracle}
\label{oracle:optim}
\textbf{Regularized MLE:} Compute a regularized MLE over $\Skew(d)$.\footnote{In our case, as both the negative log-likelihood and nuclear norm penalty are convex, this is a convex optimization. This can be implemented by simply reparametrizing $\bm\Theta \in \Skew(d)$ as $\frac{\bm\Theta' - \bm\Theta'^\top}{2}$ for unconstrained $\bm\Theta' \in \sR^{d \times d}$.}
\end{oracle}
\begin{oracle}
\label{oracle:NE}
\textbf{Population NE:} Given $\bm\Theta$, output population \SNE: $\argmax_{\pi^1}\min_{\pi^2} J_{\eta}(\pi^1, \pi^2; \bm\Theta)$.
\end{oracle}
The last oracle has been considered in prior online RLHF literature~\citep{ye2024general,wu2025greedy} and learning in regularized games under bandit feedback~\citep{yang2025incentivize,nayak2025logarithmic}.

\subsection{\texorpdfstring{Polylogarithmic Regret via Greedy Sampling: $\tilde{\gO}(\eta (\log T)^2 \wedge \sqrt{T})$}{Polylogarithmic Regret via Greedy Sampling}}
\label{sec:regret-logarithmic}

In this section, we assume the link function is logistic, $\mu(z) = (1 + e^{-z})^{-1}$, rendering the generalized linear model (GLM) well-specified as Bernoulli that admits a tight confidence sequence~\citep[Theorem 3.2]{lee2024glm}. We discuss extensions to generic link functions $\mu$ that preserve the dependencies on $d$ and $T$ in Remark~\ref{rmk:beyond-logistic}.
Additionally, we temporarily set aside the low-rank structure of $\bm\Theta_\star$ to focus purely on achieving polylogarithmic regret; we will revisit rank-exploitation in \cref{sec:regret-high-dim} to improve upon $d$ dependencies at the cost of obtaining $\sqrt{T}$ regret.

We demonstrate that a surprisingly simple algorithm, \texttt{Greedy Sampling (GS)} (Algorithm~\ref{alg:greedy} in Appendix~\ref{app:pseudocodes}), is sufficient to obtain $\tilde{\gO}(\eta (\log T)^2 \wedge \sqrt{T})$. Under \texttt{GS}, the max-player perpetually plays the greedy NE policy with respect to the current MLE $\widehat{\bm\Theta}_t$, while the min-player explores using the coverage policy $\rho$ (\cref{ass:REC}).
We show that \texttt{GS} attains the following polylogarithmic regret:

\begin{theorem}
\label{thm:regularized-log}
    Let $\delta \in (0, 1)$ and suppose that $d^2 \log\frac{T}{d} \gtrsim \kappa^{-1} \log\frac{1}{\delta}$.
    Then, with probability at least $1 - \delta,$ \texttt{GS} simultaneously attains the following bounds:
    \begin{equation}
        \MBRReg_{\eta}(T) \lesssim \min\left\{ \eta \beta \kappa^{-1} d^4 C_{\min}^{-1} \left( \log\frac{T}{d} \right)^2, \ \kappa^{-\frac{1}{2}} C_{\min}^{-\frac{1}{2}} d^2 \sqrt{T} \log\frac{T}{d} \right\}.
    \end{equation}
\end{theorem}
\begin{proof}[Proof Sketch]
    We provide a high-level sketch of the proof here; the full detail is deferred to Appendix~\ref{app:regularized-proof}. The analysis proceeds in three main steps:
    
    \paragraph{\emph{1. From Regret to Sum of Squared Errors.}}
    We begin with \cref{thm:regularized}, which bounds the instantaneous dual gap by the \emph{squared} estimation error: $\DGap({\color{blue}\hat{\pi}_t}) \lesssim \E_{{\color{blue}\bm\phi \sim \hat{\pi}_t}}[\|\mE_t \bm\phi\|_2^2]$.
    Using a standard basis decomposition, Cauchy-Schwarz with respect to the regularized Hessian of the log-likelihood loss $\gL_t(\cdot)$ at each time $t$ (see Appendix~\ref{app:regularized-proof} for the full definition), $\widehat{\mH}_t \triangleq \mI_{d^2} + \nabla^2 \gL_t(\hat{\vtheta}_t) \in \sR^{d^2 \times d^2}$, and the confidence sequence for the constrained MLE~\citep[Theorem 3.2]{lee2024glm}, we further bound this error by the sum of \emph{expected} elliptical potentials. Here, one side is the standard basis ${\color{red}\ve_j}$ and the other is chosen by ${\color{blue}\hat{\pi}_t}$:
    \begin{equation}
        \DGap({\color{blue}\hat{\pi}_t}) \lesssim (d^2 \log T) \sum_{j=1}^d \E_{{\color{blue}\bm\phi \sim \hat{\pi}_t}}\left[ \bignorm{{\color{blue}\bm\phi} \otimes {\color{red}\ve_j}}_{\widehat{\mH}_t^{-1}}^2 \right].
    \end{equation}
    
    \paragraph{\emph{2. Towards Expected Elliptical Potentials.}}
    A discrepancy arises on the right-hand side: the upper bound involves a sum over fixed \emph{basis} vectors on one side, whereas the empirical Hessian aggregates the \emph{played} features on \emph{both} sides. More specifically, $\widehat{\mH}_t$ is the (weighted) sum of outer products of $\mathrm{vec}({\color{blue}\bm\phi_t} {\color{red}\tilde{\bm\phi}_t}^\top)$, while the term inside the Mahalanobis norm is $\mathrm{vec}({\color{blue}\bm\phi_t} {\color{red}\ve_j}^\top)$. We resolve this via our \textit{\textbf{\color{teal}Coverage Lemma}} (Lemma~\ref{lem:coverage}), which leverages the feature diversity assumption (\cref{ass:REC}) to ``transform'' ${\color{red}\ve_j}$ to ${\color{red}\tilde{\bm\phi}_t}$ at the cost of $C_{\min}^{-1}$. Then, $\sum_t \DGap({\color{blue}\hat{\pi}_t})$ is bounded by:
    \begin{equation}
        \sum_{t=1}^T \DGap({\color{blue}\hat{\pi}_t}) \lesssim (d^2 \log T) \kappa^{-1} C_{\min}^{-1} \underbrace{\sum_{t=1}^T \E_{{\color{blue}\bm\phi_t \sim \hat{\pi}_t}, {\color{red}\tilde{\bm\phi}_t \sim \rho}}\left[ \bignorm{\mathrm{vec}({\color{blue}\bm\phi_t} {\color{red}\tilde{\bm\phi}_t}^\top)}_{\mV_t^{-1}}^2 \right]}_{\triangleq {\color{brown} S_T}},
    \end{equation}
    where $\mV_t \triangleq \mI_{d^2} + \sum_{s=1}^{t-1} \mathrm{vec}({\color{blue}\bm\phi_s} {\color{red}\tilde{\bm\phi}_s}^\top) \mathrm{vec}({\color{blue}\bm\phi_s} {\color{red}\tilde{\bm\phi}_s}^\top)^\top$.
    
    \paragraph{\emph{3. Martingale Concentration for Realized Variance.}}
    The final challenge is to bound the sum of \emph{expected} elliptical potentials ${\color{brown} S_T}$. The standard Elliptical Potential Lemma~\citep[Lemma 11]{abbasiyadkori2011linear} controls the sum of \emph{realized} potentials, and thus cannot be applied directly.
    To bridge this gap, we decompose the term into the realized sum plus a \emph{sum of martingale differences} as follows: denoting $\vv_t \coloneq \mathrm{vec}({\color{blue}\bm\phi_s} {\color{red}\tilde{\bm\phi}_s}^\top),$
    \begin{equation}
        {\color{brown} S_T} = \underbrace{\sum_{t=1}^T \bignorm{\vv_t}_{\mV_t^{-1}}^2}_{(a)} + \underbrace{\sum_{t=1}^T \left\{ \E_{t-1}[\bignorm{\vv_t}_{\mV_t^{-1}}^2] - \bignorm{\vv_t}_{\mV_t^{-1}}^2 \right\}}_{(b)}.
    \end{equation}
    As mentioned, $(a)$ is bounded by the standard Elliptical Potential Lemma. Bounding $(b)$ represents another technical novelty. We first apply an empirical Freedman's inequality~\citep{freedman,beygelzimer2011contextual,lee2024logistic}, which bounds $(b)$ as $(b) \lesssim A {\color{brown} S_T}$ for some constant $A \in (0, 1)$. This leads to a \textit{\textbf{\color{teal}linear self-bounding inequality}} of the form ${\color{brown} S_T} \leq A {\color{brown} S_T} + \tilde{\mathcal{O}}(d^2 \log T)$. Solving this recursion yields the final $\tilde{\mathcal{O}}(d^4 (\log T)^2)$ regret.
    
    Lastly, $\tilde{\gO}(d^2 \sqrt{T})$ is similarly obtained via the $\eta$-free (linear) error bound of Theorem~\ref{thm:regularized}.
\end{proof}

\paragraph{Discussions.}
While the quadratic error bound (\Cref{thm:regularized}) plays a critical role in achieving our results, we highlight two additional technical contributions stemming from our regret analysis.

First, we clarify the theoretical trade-offs compared to recent literature. \citet{wu2025greedy} established an $\tilde{\gO}(e^{9\eta} (\log T)^2)$ regret bound for \texttt{GS}\footnote{Technically, their algorithm is intractable under \GBPM, as it requires minimizing the cross-entropy loss over the constrained space $\Skew(d, 2r; S)$.} \emph{without} requiring a feature coverage assumption. However, their bound scales exponentially with $\eta$, is restricted exclusively to reverse KL regularization ($\psi(\cdot) = \KL(\cdot, \pi_{\mathrm{ref}})$), and fails to yield a meaningful guarantee as $\eta \to \infty$ (where it should technically recover the standard $\sqrt{T}$ unregularized regret; see \cref{app:wu} for a detailed discussion). In contrast, our analysis utilizes the coverage assumption to completely eliminate the $e^{\gO(\eta)}$ penalty, trading it for the geometric, regularization-independent quantity $C_{\min}^{-1}$.
Furthermore, echoing the spirit of \citet[Theorem 2.1]{nayak2025logarithmic}, our regret bound gracefully adapts to both regimes: it yields an $e^{\gO(\eta)}$-free $(\log T)^2$ rate for small $\eta$, and seamlessly transitions to $\sqrt{T}$ as $\eta$ increases.
In \Cref{app:experiment}, we illustrate our theoretical regret bounds with preliminary numerical experiments.

Second, our martingale concentration technique offers a streamlined alternative to recent analyses.
To handle a similar expected potential sum in the KL-regularized multi-armed bandit scenario, \citet{ji2026kl} employ a peeling argument that inherently incurs an additional doubly logarithmic terms.
By utilizing Freedman's inequality to establish a \textit{\textbf{\color{teal}linear self-bounding inequality}}, our approach is strictly tighter and cleanly avoids any such logarithmic artifacts.

\paragraph{Relations to Prior Works.}
Prior online RLHF analyses have predominantly relied on the KL-specific properties of the exponential family (e.g., bounded log-density ratios) to enable oracle reductions to least squares regression~\citep{prediction-learning-games, foster2020beyond, zhang2022thompson}; we provide a rigorous instantiation of \citet{wu2025greedy} to \GBPM~in \cref{app:wu} to illustrate this reliance. While generic regularized formulations have appeared in recent literature~\citep{tang2025rspo}, rigorous statistical guarantees for them have remained absent. 

Crucially, our analysis unifies these diverse regularizers in the online setting. We demonstrate that the specific geometry of KL is not strictly necessary for fast rates; rather, it is the \emph{strong convexity} of the regularizer that drives our results. This broadens the theoretical horizon to encompass a wide array of regularizers, including the sum of reverse KLs~\citep{le2025multiple, aminian2025multiple}, Shannon entropy~\citep{mckelvey-palfrey, mertikopoulos-sandholm, cen2024entropic}, Tsallis entropy~\citep{tsallis1988, lee2018tsallis, yang2019regularized, zimmert-seldin}, the $\chi^2$-divergence~\citep{huang2025chi-squared}, and general $f$-divergences~\citep{liese2006f-divergence, go2023f-divergence, wang2024f-divergence, xu2025f-divergence}.

Very recently, \citet{zhang2026game} demonstrated that for KL-regularized zero-sum games, greedy sampling---computing a least squares estimate and performing equilibrium computation without any explicit pessimism---suffices to achieve fast rates.
Their analysis also bounds the instantaneous duality gap by the \emph{squared} $\ell_1$-distance between policies using strong convexity, though it inherently relies on the explicit closed-form softmax structure of the KL-regularized best responses.
An interesting future direction is whether our techniques can be combined with theirs to establish similar pessimism-free fast rates for \emph{generally} regularized zero-sum games in the offline scenario.

Finally, we acknowledge recent works establishing the distinct statistical properties of different regularizers in the context of offline (reward-based) RL~\citep{jiang-xie, huang2025chi-squared, zhao2025f-divergence}. For instance, \citet{huang2025chi-squared} demonstrated that the $\chi^2$-divergence permits single-policy concentratability (unlike KL), and \citet{zhao2025f-divergence} showed that strongly convex $f$-divergences can achieve fast rates \emph{without} single-policy concentratability.
We leave the extension of these divergence-specific offline RLHF nuances within the \GBPM~framework to future work.

\begin{remark}[Beyond Parametric Models]
\label{rmk:beyond-logistic}
    A bandit problem assuming a specific parametric distribution (e.g., GLM) for the reward is known as a parametric bandit~\citep{filippi2010glm}. For semi-parametric settings—where we only assume $r_t = \mu(\langle \bm\theta_\star, \bm\phi_t \rangle) + \varepsilon_t$ with bounded noise $\varepsilon_t$—one can adopt the maximum quasi-likelihood estimator approach of \citet[Lemma 3]{li2017glm}, dating back to \citet{chen1999quasi}. This strategy utilizes a $T$-independent warm-up phase (random sampling) to ensure the design matrix is sufficiently well-conditioned, preserving the dependencies on $d$ and $T$.
\end{remark}

\subsection{\texorpdfstring{$\poly(d)$-free Regret via Explore-Then-Commit: $\widetilde{\gO}(\sqrt{\eta r T} \wedge r^{1/3} T^{2/3})$}{Poly(d)-free Regret via Explore-Then-Commit}}
\label{sec:regret-high-dim}

We now ask: what statistical gains can we achieve by maximally exploiting the \emph{low-rank structure} of $\bm\Theta_\star$? This question is particularly relevant in the \emph{high-dimensional} or \emph{data-poor} regime, where $T$ is not sufficiently large relative to $d$ (specifically, $d^{c_1} \lesssim T \lesssim d^{c_2}$ for constants $0 < c_1 < c_2$). 
Such regimes are characteristic of modern applications involving high-dimensional features~\citep{tucker2020highdim,li2024highdim}.

Here, it is imperative to avoid explicit $\poly(d)$ dependencies in the regret bound, isolating the complexity to the intrinsic rank $r$ and unavoidable dependencies on the feature coverage $C_{\min}^{-1}$ (\cref{ass:REC})~\citep[Table 1]{zeng2025lowerbound}, which is known to be unavoidable.
To achieve this, we leverage the low-rank structure of $\bm\Theta_\star$ using techniques standard in high-dimensional bandits~\citep{carpentier2012sparse,hao2020sparse,kim2019lasso,oh2021lasso,li2022unified,lu2021generalized,kang2022generalized,jang2022popart,jang2024lowpopart}. In this regime, sufficient initial exploration is requisite to identify and exploit the underlying low-rank subspace.

Following standard approaches in high-dimensional contextual bandits~\citep{hao2020sparse,li2022unified,jang2024lowpopart}, we employ the \texttt{Explore-Then-Commit (ETC)} algorithm. The players explore for $T_0$ rounds using the coverage policy $\rho$, compute a symmetric Nash equilibrium (SNE) based on the resulting \emph{nuclear-norm regularized MLE}~\citep{fan2019generalized,lee2025gl-lowpopart}, and symmetrically commit to this policy for the remaining rounds (see Algorithm~\ref{alg:etc} in \cref{app:pseudocodes}).

We now present the regret bound for \texttt{ETC}, with the full proof deferred to \cref{app:etc}, which also relies critically on \Cref{thm:regularized}:
\begin{theorem}
\label{thm:regularized-sqrt}
    Let $\delta \in (0, 1)$, and suppose $T \gtrsim \kappa^{-2} C_{\min}^{-4} d r \log\frac{d}{\delta}$.
    By setting the regularization parameter $\lambda_{T_0} = \sqrt{\frac{32 L_\mu \log(4d/\delta)}{T_0}}$, \texttt{ETC} attains $\MBRReg_{\eta}(T) \lesssim T_0$ with probability at least $1 - \delta$. Specifically, depending on the choice of $T_0$, \texttt{ETC} achieves the following bounds:
    \begin{equation}
        \MBRReg_{\eta}(T) \lesssim 
        \begin{cases} 
            \kappa^{-1} C_{\min}^{-2} \sqrt{T \eta \beta r \log\frac{d}{\delta}} & \text{if } T_0 \asymp \kappa^{-1} C_{\min}^{-2} \sqrt{T \eta \beta r \log\frac{d}{\delta}}, \\
            \left( \kappa^{-2} C_{\min}^{-4} r T^2 \log\frac{d}{\delta} \right)^{1/3} & \text{if } T_0 \asymp \left( \kappa^{-2} C_{\min}^{-4} r T^2 \log\frac{d}{\delta} \right)^{1/3}.
        \end{cases}
    \end{equation}
\end{theorem}

\paragraph{Discussions.}
This result highlights two critical theoretical insights for the high-dimensional regime.
First, thanks to the quadratic error bound established in \cref{thm:regularized}, \texttt{ETC} is able to achieve a fast $\tilde{\gO}(\sqrt{\eta r T})$ rate.
Interestingly, this surpasses the $\tilde{\gO}(T^{2/3})$ rate typically associated with \texttt{ETC} algorithms~\citep{banditalgorithms}. The tightness of these bounds depends directly on the regularization coefficient $\eta$; specifically, focusing on $\eta, d,$ and $T$, the $\tilde{\gO}(\sqrt{\eta r T})$ bound is asymptotically tighter than the $T^{2/3}$ regret whenever $\eta \lesssim (T / \log d)^{1/3}$.
Second, both regrets in \cref{thm:regularized-sqrt} explicitly scale with $r$ rather than $d$.
This confirms that our framework can effectively exploit the low-rank structure of general preference games.

\section{Conclusion}
\label{sec:conclusion}

In this work, we investigated regularized max-regret minimization under the GBPM with bandit feedback, utilizing it as a theoretical abstraction to analyze the statistical impact of general regularizers beyond reverse KL.
Under this framework, we demonstrated that ``fast'' regret rates are \emph{not} an exclusive artifact of KL-geometry, but can be achieved for \emph{any} strongly convex regularizer.
Specifically, we first established a novel quadratic error bound on the dual gap that is central to our subsequent results.
Its proof, which combines the strong convexity of the regularizer with the skew-symmetry of the preference matrix and the IPM representation of the $\ell_1$-distance, stands as a key technical contribution that may be of independent interest.
Armed with this crucial theorem, we showed that under a feature diversity assumption (\Cref{ass:REC}), \texttt{Greedy Sampling} and \texttt{Explore-Then-Commit} yield polylogarithmic and $\poly(d)$-free (up to $C_{\min}^{-1}$) regret, respectively.
We detail further future directions in \cref{app:future}.

\bibliographystyle{plainnat}
\bibliography{references}

\newpage
\appendix

\crefalias{section}{appendix}
\crefname{appendix}{Appendix}{Appendices}
\Crefname{appendix}{Appendix}{Appendices}

\onecolumn
\tableofcontents
\vspace{0.5cm}

\newpage
\section{Pseudocodes for \texttt{Greedy Sampling} and \texttt{Explore-Then-Commit}}
\label{app:pseudocodes}

\begin{algorithm2e}[!h]
\caption{\texttt{Greedy Sampling}}
\label{alg:greedy}
\textbf{Input:} Exploration policy $\rho$\;
\BlankLine

Initialize $\hat{\pi}_1 \gets \rho$\;

\For{$t = 1, 2, \cdots, T$}{
    Observe $\vx_t \sim d_0$\;
    
    Sample ${\color{blue} \va_t^1} \sim \hat{\pi}_t(\cdot | \vx_t)$ and ${\color{red} \va_t^2} \sim \rho(\cdot | \vx_t)$\;
    
    Observe $r_t := \indicator[{\color{blue}\va_t^1} \succ {\color{red}\va_t^2}] \sim \mathrm{Ber}(\mu({\color{blue}\phi_t^1}^\top \bm\Theta_\star {\color{red}\phi_t^2}) \mid \vx_t)$, where $\phi_t^i := \phi(\vx_t, \va_t^i)$\;

    Compute an estimator $\widehat{\bm\Theta}_{t+1} \in \Skew(d)$\;

    Compute a (symmetric) Nash equilibrium:
    $\hat{\pi}_{t+1} \gets \argmax_{\pi^1 \in \Pi} \min_{\pi^2 \in \Pi} J_{\eta}(\pi^1, \pi^2; \widehat{\bm\Theta}_{t+1}).$\;
}
\end{algorithm2e}

\begin{algorithm2e}[!h]
\caption{\texttt{Explore-Then-Commit}}
\label{alg:etc}
\textbf{Input:} Exploration policy $\rho$ and budget $T_0$\;
\BlankLine

\For{$t = 1, 2, \cdots, T_0$}{
    Observe $\vx_t \sim d_0$\;
    
    Sample ${\color{blue} \va_t^1} \sim \rho(\cdot | \vx_t)$ and ${\color{red} \va_t^2} \sim \rho(\cdot | \vx_t)$\;
    
    Observe $r_t := \indicator[{\color{blue}\va_t^1} \succ {\color{red}\va_t^2}] \sim \mathrm{Ber}(\mu({\color{blue}\phi_t^1}^\top \bm\Theta_\star {\color{red}\phi_t^2}) \mid \vx_t)$, where $\phi_t^i := \phi(\vx_t, \va_t^i)$\;
}

Compute an estimator $\widehat{\bm\Theta} \in \Skew(d)$\;

Compute a (symmetric) Nash equilibrium:
$\hat{\pi} \gets \argmax_{\pi^1 \in \Pi} \min_{\pi^2 \in \Pi} J_{\eta}(\pi^1, \pi^2; \widehat{\bm\Theta}) .$\;

\For{$t = T_0+1, \cdots, T$}{
    Symmetrically commit to $(\hat{\pi}, \hat{\pi})$\;
}
\end{algorithm2e}

\newpage
\section{\texorpdfstring{Proof of Lemma~\ref{lem:strongly-convex}}{Proof of Lemma 3.4}}
\label{app:convexity}

For each fixed context $\vx \in \gX$, $J(\hat{\pi}, \pi | \vx) = \E_{\va \sim \hat{\pi}(\cdot | \vx), \vb \sim \pi(\cdot | \vx)}[\mu(\va^\top \Theta \vb)]$ is linear in $\pi(\cdot | \vx)$.
Because the sum of a linear function and a strongly convex function remains strongly convex, the regularized objective $J_{\eta}(\hat{\pi}, \pi | x)$ is $(\eta \beta)^{-1}$-strongly convex in $\ell_1$.
By the zeroth-order characterization of strong convexity~\citep[Theorem 2.1.9]{nesterov}, for any $\pi, \tilde{\pi} \in \Pi$ and $\alpha \in [0, 1]$, the following holds point-wise for any $\vx$:
\begin{equation}
    J_{\eta}(\hat{\pi}, \alpha \pi + (1-\alpha) \tilde{\pi} | \vx) \le \alpha J_{\eta}(\hat{\pi}, \pi | \vx) + (1-\alpha) J_{\eta}(\hat{\pi}, \tilde{\pi} | \vx) - \frac{\alpha(1-\alpha)}{2 \eta \beta} \|\pi(\cdot | \vx) - \tilde{\pi}(\cdot | \vx) \|_1^2
\end{equation}
Taking the expectation over $\vx \sim d_0$ on both sides, choosing $\pi = \hat{\pi}$, and denoting the global objective as $F(\pi) := \mathbb{E}_{\vx \sim d_0}[J_{\eta}(\hat{\pi}, \pi | \vx)]$ for simplicity, we obtain:
\begin{equation}
    F(\alpha \pi + (1-\alpha) \tilde{\pi}) \le \alpha F(\pi) + (1-\alpha) F(\tilde{\pi}) - \frac{\alpha(1-\alpha)}{2 \eta \beta} \E_{\vx \sim d_0}\left[ \|\pi(\cdot | \vx) - \tilde{\pi}(\cdot | \vx) \|_1^2 \right],
\end{equation}
Choosing $\tilde{\pi} = \argmin_{\pi \in \Pi} F(\pi)$, we have that $F(\tilde{\pi}) \leq F(\alpha \hat{\pi} + (1-\alpha) \tilde{\pi}),$ for \emph{any} $\alpha \in [0, 1]$.
Rearranging the inequality and dividing both sides by $\alpha,$ the following holds for any $\alpha \in (0, 1]$:
\begin{equation}
    \underbrace{\E_{\vx \sim d_0}[J_{\eta}(\hat{\pi}, \hat{\pi} \mid \vx)]}_{= \frac{1}{2}} - E_{\vx \sim d_0}[J_{\eta}(\hat{\pi}, \tilde{\pi} \mid \vx)]
    = F(\hat{\pi}) - F(\tilde{\pi}) \ge \frac{1 - \alpha}{2 \eta \beta} \E_{\vx \sim d_0}\left[ \|\pi(\cdot | \vx) - \tilde{\pi}(\cdot | \vx) \|_1^2 \right].
\end{equation}
We then conclude by taking the limit $\alpha \rightarrow 0^+.$
\qed

\newpage

\section{\texorpdfstring{Statistical Cost of Obtaining $\rho$}{Statistical Cost of Obtaining rho}}
\label{app:rho}

Let us define the population E-optimal design:
\begin{equation}
    \rho^\star \gets \argmax_{\rho : \gX \rightarrow \Delta(\gA)} \left\{ L(\rho) \triangleq \lambda_{\min}\left( \E_{\vx \sim d_0}\E_{\va \sim \rho(\cdot | \vx)}\left[ \bm\phi(\vx, \va) \bm\phi(\vx, \va)^\top \right] \right) \right\},
\end{equation}
and let us denote $C_{\min} := L(\rho^\star)$.
Recall that $\bignorm{\bm\phi(\vx, \va)}_2 \leq 1$, always.

\newcommand{\Rho}{\mathrm{P}}
\paragraph{Offline Scenario.}
Suppose that we have a $\{\vx_i\}_{i \in [N]}$ with $\vx_i \sim d_0$ i.i.d., and suppose that we have access to an exploratory policy class $\Rho \subset \{ \rho : \gX \rightarrow \Delta(\gA) \}$ that satisfies realizability, i.e., $\rho^\star \in \Rho$.
For simplicity, suppose that $|\Rho| < \infty$, as if not, then one should be able to extend the arguments using standard covering and uniform convergence arguments, provided that $\Rho$ has a finite complexity measure.

We define the empirical E-optimal design:
\begin{equation}
    \hat{\rho}_N \gets \argmax_{\rho : \gX \rightarrow \Delta(\gA)} \left\{ \hat{L}_N(\rho) \triangleq \lambda_{\min}\left( \frac{1}{N} \sum_{i=1}^N \E_{\va \sim \rho(\cdot | \vx_i)}\left[ \bm\phi(\vx_i, \va) \bm\phi(\vx_i, \va)^\top \right] \right) \right\}
\end{equation}
Then we have the following guarantee:
\begin{proposition}
    $\sP\left( L(\hat{\rho}_N) \geq C_{\min} / 2 \right) \geq 1 - \delta$, provided that $N \geq 32 C_{\min}^{-2} \log\frac{2 d |\Pi|}{\delta}$.
\end{proposition}
\begin{proof}
    Note that each matrix $\mZ_i(\rho) \triangleq \E_{\vx \sim d_0}\E_{\va \sim \rho(\cdot | \vx)}\left[ \bm\phi(\vx, \va) \bm\phi(\vx, \va)^\top \right] - \E_{\va \sim \rho(\cdot | \vx_i)}\left[ \bm\phi(\vx_i, \va) \bm\phi(\vx_i, \va)^\top \right]$ is i.i.d. that satisfies $\E[\mZ_i(\rho)] = 0$ and $\mZ_i(\rho)^2 \preceq \mI_d$.\footnote{This is because for $\vzero \preceq \mA, \mB \preceq \mI_d$, $-\mI_d \preceq \mA - \mB \preceq \mI_d$, which then implies that $(\mA - \mB)^2 \preceq \mI_d.$}
    Thus, invoking matrix Hoeffding~\citep[Theorem 1.3 \& Remark 7.4]{tropp2012user-friendly}, we have that \emph{for each} $\rho \in \Rho,$ with probability at least $2d \exp\left(-\frac{N \varepsilon^2}{2} \right),$
    \begin{equation}
        \bignormop{\E_{\vx \sim d_0}\E_{\va \sim \rho(\cdot | \vx)}\left[ \bm\phi(\vx, \va) \bm\phi(\vx, \va)^\top \right] - \frac{1}{N} \sum_{i=1}^N \E_{\va \sim \rho(\cdot | \vx_i)}\left[ \bm\phi(\vx_i, \va) \bm\phi(\vx_i, \va)^\top \right]} \geq \varepsilon.
    \end{equation}
    Union bound over $\rho \in \Rho$ and rearranging, we have that with probability at least $1 - \delta,$
    \begin{equation}
        \gE_N \triangleq \sup_{\rho \in \Rho} \bignormop{\frac{1}{N} \sum_{i=1}^N \mZ_i(\rho)} \leq \sqrt{\frac{2}{N} \log\frac{2 d|\mathrm{P}|}{\delta}}.
    \end{equation}
    Then, we have that with probability at least $1 - \delta$,
    \begin{align}
        L(\hat{\rho}_N) &= L(\rho^\star) + \underbrace{\hat{L}_N(\rho^\star) - L(\rho^\star)}_{\geq -\gE_N} + \underbrace{L(\hat{\rho}_N) - \hat{L}_N(\hat{\rho}_N)}_{\geq -\gE_N} + \underbrace{\hat{L}_N(\hat{\rho}_N) - \hat{L}_N(\rho^\star)}_{\geq 0} \\
        &\geq C_{\min} - 2 \gE_N \tag{Weyl's inequality, $\hat{\rho}_N = \argmax_\rho \hat{L}_N(\rho)$} \\
        &\geq C_{\min} - \sqrt{\frac{8}{N} \log\frac{2 d|\mathrm{P}|}{\delta}} \geq \frac{C_{\min}}{2},
    \end{align}
    and thus, the statement follows.
\end{proof}

\paragraph{Online Scenario.}
This is largely inspired by the ``design static or nonadaptive policies that can be used to gather data to identify optimal contextualized decision policies'' as in \citet{zanette2021design}.

Suppose that we interact with the environment for $T_0$ iterations via the following elliptical exploration algorithm:
\begin{enumerate}
    \item \textbf{Initialize:} $\mV_1 = \lambda \mI_d$ for some regularization parameter $\lambda > 0$.
    \item \textbf{For} $t = 1, \dots, T_0$:
    \begin{enumerate}
        \item Observe context $\vx_t \sim d_0$.
        \item Define the exploratory policy $\rho_t(\vx_t)$ to \emph{deterministically} select the action maximizing the elliptical bonus:
        \begin{equation}
            \rho_t(\vx) \coloneq \argmax_{\va \in \gA} \bignorm{\bm\phi(\vx, \va)}_{\mV_{t-1}^{-1}}^2.
        \end{equation}
        \item Choose action $\va_t = \rho_t(\vx_t)$ and update the covariance matrix: $\mV_{t+1} = \mV_t + \bm\phi(\vx_t, \va_t)\bm\phi(\vx_t, \va_t)^\top$.
    \end{enumerate}
    \item \textbf{Return:} The uniform mixture policy $\bar{\rho}_{T_0} \triangleq \frac{1}{T_0} \sum_{t=1}^{T_0} \rho_t$.
\end{enumerate}

Then we have the following guarantee:
\begin{proposition}
    Set $T_0 \asymp \frac{d^2}{C_{\min}^2} \log\frac{d}{C_{\min}^2}\log\frac{d}{\delta}$ and $\lambda = 1$.
    Then, the above elliptical exploration algorithm guarantees the following:
    \begin{equation}
        \sP\left( L(\bar{\rho}_{T_0}) \gtrsim \frac{C_{\min}}{d \log\frac{d}{C_{\min}^2} \log\frac{1}{\delta}} \right) \geq 1 - \delta.
    \end{equation}
\end{proposition}
\begin{proof}
    The proof proceeds similarly to our Freedman-based bounding of expected elliptical potentials.
    
    Let $\Sigma(\rho) \triangleq \E_{\vx \sim d_0}\E_{\va \sim \rho(\cdot \mid \vx)}[\bm\phi(\vx, \va) \bm\phi(\vx, \va)^\top]$, and $M_t$ be the conditional expectation of the elliptical bonus, conditioned on the history $\gF_{t-1} = \sigma(\vx_1, \va_1, \cdots, \vx_{t-1}, \va_{t-1})$:
    \begin{equation}
        M_t \triangleq \E_{\vx \sim d_0}\left[ \max_{\va \in \gA} \bignorm{\bm\phi(\vx, \va)}_{\mV_t^{-1}}^2 \Bigm| \gF_{t-1} \right] \leq 1.
    \end{equation}
    Because $\max \geq \E$, it upper bounds the expectation over the optimal $\rho^\star$:
    \begin{align}
        M_t &\geq \E_{\vx \sim d_0}\E_{\va \sim \rho^\star(\cdot \mid \vx)}\left[ \bignorm{\bm\phi(\vx, \va)}_{\mV_t^{-1}}^2 \Bigm| \gF_{t-1} \right]
        \overset{(*)}{=} \tr\left( \mV_t^{-1} \Sigma(\rho^\star) \right) \geq C_{\min} \tr(\mV_t^{-1}),
    \end{align}
    where $(*)$ follows from linearity of the expectation.

    Because $\lambda = 1$ and $\|\bm\phi\|_2 \leq 1$, we have $y_t \triangleq \bignorm{\bm\phi(\vx_t, \va_t)}_{\mV_t^{-1}}^2 \in [0, 1]$.
    Thus, by the elliptical potential lemma (\Cref{lem:EPL}), we have that
    \begin{equation}
        \sum_{t=1}^{T_0} y_t \leq 2d \log\left( 1 + \frac{T_0}{d} \right),
    \end{equation}
    where the last inequality follows from our choice of $\lambda = 1$.

    By the same variant of the Freedman's inequality (\Cref{lem:freedman}), for any $\xi \in (0, 1]$, the following holds with probability at least $1 - \frac{\delta}{2}$:
    \begin{equation}
        \sum_{t=1}^{T_0} (M_t - y_t) \leq (e - 2) \xi \sum_{t=1}^{T_0} \underbrace{\E[(M_t - y_t)^2 \mid \gF_{t-1}]}_{(*)} + \frac{1}{\xi} \log\frac{2}{\delta}.
    \end{equation}
    We bound $(*)$ as follows: denoting $\E_{t-1} \triangleq \E[\cdot \mid \gF_{t-1}]$,
    \begin{equation}
        \E_{t-1}[(M_t - y_t)^2] = \E_{t-1}[y_t^2] - M_t^2
        \leq \E_{t-1}[y_t] - M_t^2 \leq M_t.
    \end{equation}
    Choosing $\xi = \frac{1}{2(e - 2)}$, we have the following linear self-bounding inequality:
    \begin{equation}
        \sum_{t=1}^{T_0} M_t
        = \sum_{t=1}^{T_0} y_t + \sum_{t=1}^{T_0} (M_t - y_t)
        \leq 2d \log\left( 1 + \frac{T_0}{d} \right) + \frac{1}{2} \sum_{t=1}^{T_0} M_t + 2(e - 2) \log\frac{2}{\delta}.
    \end{equation}
    Combining this with the lower bound of $M_t$, we have that
    \begin{equation}
        C_{\min} \sum_{t=1}^{T_0} \tr(\mV_t^{-1}) \leq \sum_{t=1}^{T_0} M_t \leq 4 d \log\left( 1 + \frac{T_0}{d} \right) + 4 (e - 2) \log\frac{2}{\delta}.
    \end{equation}
    As $\tr(\mV_t^{-1}) \geq \frac{1}{\lambda_{\min}(\mV_t)} \geq \frac{1}{\lambda_{\min}(\mV_{T_0})}$, we have that
    \begin{equation}
        \frac{\lambda_{\min}(\mV_{T_0})}{T_0} \geq \frac{C_{\min}}{4 d \log\left( 1 + \frac{T_0}{d} \right) + 4 (e - 2) \log\frac{2}{\delta}}.
    \end{equation}

    Let $\mZ_t \triangleq \Sigma(\rho_t) - \bm\phi(\vx_t, \va_t) \bm\phi(\vx_t, \va_t)^\top$. 
    Because $\rho_t$ is fully determined (measurable) by $\gF_{t-1}$, we have that $\E[\bm\phi(\vx_t, \va_t) \bm\phi(\vx_t, \va_t)^\top \mid \gF_{t-1}] = \Sigma(\rho_t)$, i.e., $\{\mZ_t\}_{t=1}^{T_0}$ is a matrix martingale difference sequence adapted to $\gF_t$. 
    Furthermore, because $\mathbf{0} \preceq \Sigma(\rho_t) \preceq \mI$ and $\mathbf{0} \preceq \bm\phi \bm\phi^\top \preceq \mI$, the differences are bounded: $\lambda_{\max}(\mZ_t) \leq 1$ and $\lambda_{\min}(\mZ_t) \geq -1$.
    
    By the matrix Azuma inequality~\citep[Theorem 7.1, Remark 7.8]{tropp2012user-friendly}, with probability at least $1 - \frac{\delta}{2}$:
    \begin{equation}
        \lambda_{\min}\left( \sum_{t=1}^{T_0} \mZ_t \right) \geq -\sqrt{2 T_0 \log\frac{2d}{\delta}}.
    \end{equation}
    
    Notice that the population covariance of the uniform mixture is $\Sigma(\bar{\rho}_{T_0}) = \frac{1}{T_0} \sum_{t=1}^{T_0} \Sigma(\rho_t)$. We can rewrite this sum as:
    \begin{equation}
        \frac{1}{T_0} \sum_{t=1}^{T_0} \Sigma(\rho_t) = \frac{1}{T_0} \left( \mV_{T_0} - \lambda \mI + \sum_{t=1}^{T_0} \mZ_t \right).
    \end{equation}
    Taking the minimum eigenvalue on both sides, we lower bound $L(\bar{\rho}_{T_0})$ as follows:
    \begin{align}
        L(\bar{\rho}_{T_0}) &= \frac{1}{T_0} \lambda_{\min}\left( \mV_{T_0} - \mI + \sum_{t=1}^{T_0} \mZ_t \right) \\
        &= \frac{1}{T_0} \lambda_{\min}\left( \mV_{T_0} + \sum_{t=1}^{T_0} \mZ_t \right) - \frac{\lambda}{T_0} \\
        &\geq \frac{1}{T_0} \lambda_{\min}(\mV_{T_0}) + \frac{1}{T_0} \lambda_{\min}\left( \sum_{t=1}^{T_0} \mZ_t \right) - \frac{1}{T_0} \tag{$\lambda_{\min}(\cdot)$ is concave} \\
        &\geq \frac{C_{\min}}{4 d \log\left(1 + \frac{T_0}{d}\right) + 4 (e - 2) \log\frac{2}{\delta}} - \sqrt{\frac{2}{T_0} \log\frac{2d}{\delta}} - \frac{1}{T_0} \\
        &\geq \frac{C_{\min}}{4 d \log\left(1 + \frac{T_0}{d}\right) + 4 (e - 2) \log\frac{2}{\delta}} - \sqrt{\frac{8}{T_0} \log\frac{2d}{\delta}},
    \end{align}
    where the last inequality follows given that $T_0 \geq \left( 8 \log\frac{2d}{\delta} \right)^{-1}$.

    We lastly solve for a condition on $T_0$ such that the first term dominates, i.e.,
    \begin{equation}
        \frac{C_{\min}}{d \log\left(1 + \frac{T_0}{d}\right) + \log\frac{1}{\delta}} \gtrsim \sqrt{\frac{1}{T_0} \log\frac{d}{\delta}} \implies \sqrt{T_0} \gtrsim \frac{d \log\left(1 + \frac{T_0}{d}\right) + \log\frac{1}{\delta}}{C_{\min}} \sqrt{\log\frac{d}{\delta}}.
    \end{equation}
    Solving this condition yields the required exploration budget:
    \begin{equation}
        T_0 \asymp \frac{d^2}{C_{\min}^2} \log\frac{d}{C_{\min}^2}\log\frac{d}{\delta}.
    \end{equation}
    Plugging this $T_0$ back into our lower bound, we conclude that with probability at least $1 - \delta$ (via a union bound over the Freedman and Azuma events):
    \begin{equation}
        L(\bar{\rho}_{T_0}) \gtrsim \frac{C_{\min}}{d \log\frac{d}{C_{\min}^2} \log\frac{1}{\delta}},
    \end{equation}
    which completes the proof.
\end{proof}

\newpage

\section{\texorpdfstring{Proof of \cref{thm:regularized-log}: Regret Bound of \texttt{Greedy Sampling}}{Proof of Theorem 4.2: Regret Bound of Greedy Sampling}}
\label{app:regularized-proof}
\subsection{Main Proof}

We will use several properties of the Kronecker product $\otimes$ throughout the proof; see \citet{minka1997matrix} for a reference.
Let us denote $\mathrm{vec} : \sR^{d \times d} \rightarrow \sR^{d^2}$ as the (column-wise) vectorization operator, and $\mathrm{mat} \triangleq \mathrm{vec}^{-1} : \sR^{d^2} \rightarrow \sR^{d \times d}$ to be the matrization operator.

\paragraph{Part I. $\tilde{\gO}(\eta \beta (\log T)^2)$ Regret Bound.}
To mirror the proof sketch in the main text, we divide the proof into three parts.

\textbf{1. From Regret to Sum of Squared Errors.}
With $\vv_t := \mathrm{vec}(\bm\phi_t^2 (\bm\phi_t^1)^\top) \in \sR^{d^2}$, our MLE is defined as follows:
\begin{equation}
    \widehat{\bm\Theta}_t \gets \mathrm{mat}(\hat{\vtheta}_t), \quad
    \hat{\vtheta}_t \gets \argmin_{\vtheta \in \gK_S} \gL_t(\vtheta), \quad
    \gL_t(\vtheta) := \sum_{s=1}^{t-1} \left\{ m(\langle \vtheta, \vv_t \rangle) - r_t \langle \vtheta, \vv_t \rangle \right\},
\end{equation}
where $m(\cdot)$ is the log-partition function~\citep{graphicalmodels} such that $m' = \mu$, and
\begin{equation}
    \mathcal{K}_S := \left\{ \vtheta \in \mathbb{R}^{d^2} : \bignorm{\vtheta}_2 \leq S \text{ and } \mathrm{mat}(\vtheta)^\top = -\mathrm{mat}(\vtheta) \right\}.
\end{equation}
As $\mu(z) = (1 + e^{-z})^{-1}$ for our proof, this implies that $R_s = 1$ and $L_\mu = \frac{1}{4}$.

Let us denote the regularized Hessian of $\gL_t(\cdot)$ at the MLE $\hat{\vtheta}_t$ as
\begin{equation}
    \widehat{\mH}_t := \mI_{d^2} + \sum_{s=1}^{t-1} \dmu(\langle \hat{\vtheta}_s, \vv_s \rangle) \vv_s \vv_s^\top.
\end{equation}

We first recall the following elliptical confidence sequence:
\begin{lemma}[Theorem 3.2 of \citet{lee2024glm}]
\label{lem:confidence-sequence}
    For any adaptively collected\footnote{This can be formalized via the canonical bandit model as described in \citet[Chapter 4.6]{banditalgorithms}.} $\{(\vv_t, r_t)\}$ and any $\delta \in (0, 1)$ we have
    \begin{equation}
        \sP\left( \bignorm{\vtheta_\star - \hat{\vtheta}_t}_{\widehat{\mH}_t}^2 \lesssim \gamma_t(\delta), \ \ \forall t \geq 1 \right) \geq 1 - \delta,
    \end{equation}
    where $\gamma_t(\delta) \lesssim S^5 + S \log\frac{1}{\delta} + S d^2 \log\frac{S t}{d}$.
\end{lemma}
We note that this is used solely for the proof and does not affect the algorithm in any way.

From the $\eta$-dependent bound of \cref{thm:regularized}, we have that
\begin{equation}
    \MBRReg(T) \leq (4 L_\mu^2 \eta \beta + L_\mu) \sum_{t=1}^T \E_{\bm\phi_t \sim \hat{\pi}_t}\left[ \bignorm{\mE_t \bm\phi_t}_2^2 \right].
\end{equation}

We decompose the squared Euclidean norm as follows:
denoting the standard basis of $\sR^d$ as $\{\ve_j\}_{j \in [d]}$,
\begin{align}
    \bignorm{\mE_t \bm\phi_t}_2^2 &= \sum_{j=1}^d \left( \ve_j^\top \mE_t \bm\phi_t \right)^2 \\
    &= \sum_{j=1}^d \langle \mathrm{vec}(\mE_t), \mathrm{vec}(\ve_j \bm\phi_t^\top) \rangle^2 \\
    &\leq \sum_{j=1}^d \bignorm{\mathrm{vec}(\mE_t)}_{\widehat{\mH}_t}^2 \bignorm{\mathrm{vec}(\ve_j \bm\phi_t^\top)}_{\widehat{\mH}_t^{-1}}^2 \tag{Cauchy-Schwarz} \\
    &\lesssim \gamma_t(\delta)^2 \sum_{j=1}^d \bignorm{\bm\phi_t \otimes \ve_j}_{\widehat{\mH}_t^{-1}}^2. \tag{\cref{lem:confidence-sequence}}
\end{align}

\textbf{2. Towards Expected Elliptical Potentials.}

We now present our key technical lemma:

\begin{lemma}[Coverage Lemma]
\label{lem:coverage}
    Let $\tilde{\bm\phi} \sim \rho$ be such that $\E_{\tilde{\bm\phi} \sim \rho}\left[ \tilde{\bm\phi} \tilde{\bm\phi}^\top \right] \succeq C_{\min} \mI_d.$
    Then, for any positive semi-definite $\mM \in \sR^{d^2 \times d^2}$ and any vector $\bm\phi \in \sR^d$, the following holds:
    \begin{equation}
        \sum_{j=1}^d (\bm\phi \otimes \ve_j)^\top \mM (\bm\phi \otimes \ve_j) \leq C_{\min}^{-1} \E_{\tilde{\bm\phi} \sim \rho} \left[ \mathrm{vec}(\bm\phi \tilde{\bm\phi}^\top)^\top \mM \mathrm{vec}(\bm\phi \tilde{\bm\phi}^\top) \right].
    \end{equation}
\end{lemma}

From hereon, we denote $\E_{t-1}[\cdot] \triangleq \E_{\bm\phi_t \sim \hat{\pi}_t, \tilde{\bm\phi}_t \sim \rho}[\cdot]$, where $\E$ is to indicate that the expectation is conditional on the history, due to $\hat{\pi}_t$ being history-dependent.

Applying the above lemma with $\mM = \widehat{\mH}_t^{-1}$, we have:
\begin{align}
    \MBRReg(T) &\lesssim \eta \beta d^2 C_{\min}^{-1} \log\frac{T}{d} \sum_{t=1}^T \E_{t-1} \left[ \mathrm{vec}(\bm\phi_t \tilde{\bm\phi}_t^\top)^\top \widehat{\mH}_t^{-1} \mathrm{vec}(\bm\phi_t \tilde{\bm\phi}_t^\top) \right] \nonumber \\
    &\leq \eta \beta \kappa^{-1} d^2 C_{\min}^{-1} \log\frac{T}{d} \sum_{t=1}^T \E_{t-1} \left[ \mathrm{vec}(\bm\phi_t \tilde{\bm\phi}_t^\top)^\top \mV_t^{-1} \mathrm{vec}(\bm\phi_t \tilde{\bm\phi}_t^\top) \right], \label{eqn:regret-gs-1}
\end{align}
where we have bounded $\gamma_t(\delta)^2 \leq \gamma_T(\delta)^2 \lesssim d^2 \log\frac{T}{d}$, and we denote $\vv_t = \bm\phi_t \tilde{\bm\phi}_t^\top$ and $\mV_t := \frac{1}{\kappa \wedge 1} \mI + \sum_{s=1}^{t-1} \vv_s \vv_s^\top$.

\textbf{3. Martingale Concentration for Realized Variance.}

We now consider the elliptical-type quantity:
\begin{align}
    \underbrace{\sum_{t=1}^T \E_{t-1}\left[ \vv_t^\top \mV_t^{-1} \vv_t \right]}_{\triangleq {\color{brown}S_T}}
    &= \sum_{t=1}^T \vv_t^\top \mV_t^{-1} \vv_t + \sum_{t=1}^T \underbrace{\left( \E_{t-1}\left[ \vv_t^\top \mV_t^{-1} \vv_t \right] - \vv_t^\top \mV_t^{-1} \vv_t \right)}_{\triangleq M_t} \nonumber \\
    & =  \underbrace{\sum_{t=1}^T \vv_t^\top \mV_t^{-1} \vv_t}_{(a)} + \underbrace{\sum_{t=1}^T M_t}_{(b)}. \label{eqn:elliptical}
\end{align}

We bound $(a)$ via the usual elliptical potential lemma (\cref{lem:EPL}):
\begin{equation}
    \sum_{t=1}^T \vv_t^\top \mV_t^{-1} \vv_t \leq 2d^2 \log\left( 1 + \frac{\kappa T}{d^2} \right).
\end{equation}

We now bound $(b)$.
The crucial observation is that $M_t$ is a martingale difference sequence that satisfies $\max_{t \geq 1} |M_t| \leq \kappa$, which prompts us to use the following variant of Freedman's inequality~\citep{freedman,beygelzimer2011contextual}:
\begin{lemma}[Lemma 3 of \citet{lee2024logistic}]
\label{lem:freedman}
    Let $M_t$ be a martingale difference sequence that satisfies $\max_{t \geq 1} |M_t| \leq R$.
    Then for any $\delta \in (0, 1)$ and $\xi \in (0, 1/R]$, we have:
    \begin{equation}
        \sP\left( \sum_{s=1}^t M_s \leq (e - 2) \xi \sum_{s=1}^t \E_{s-1}[M_s^2] + \frac{1}{\xi} \log\frac{2}{\delta}, \quad \forall t \geq 1 \right) \geq 1 - \frac{\delta}{2}.
    \end{equation}
\end{lemma}

We will now show, with high probability, that the variance term induces a linear self-bounding inequality.
Denoting ${\color{brown}Z_s} := \vv_s^\top \mV_s^{-1} \vv_s$, which satisfies $0 \leq {\color{brown} Z_s} \leq \kappa \wedge 1$,
\begin{equation}
    \E_{s-1}[M_s^2] \leq \E_{s-1}[{\color{brown}Z_s}^2] \leq (\kappa \wedge 1) \E_{s-1}[{\color{brown}Z_s}],
\end{equation}
where the first inequality follows from the fact that for any random variable $X$, $\mathrm{Var}[X] \leq \E[X^2]$.

Choosing $\xi = \frac{\kappa \wedge 1}{2 (e - 2)} \leq \kappa$ in \cref{lem:freedman}, the following holds with probability at least $1 - \frac{\delta}{2}$:
\begin{equation}
    (b) = \sum_{t=1}^T M_t \leq \frac{(\kappa \wedge 1)^2}{2} \underbrace{\sum_{t=1}^T \E_{t-1}[{\color{brown}Z_t}]}_{={\color{brown}S_T} \text{!}} + \frac{2(e - 2)}{\kappa \wedge 1} \log\frac{2}{\delta}.
\end{equation}

Now bringing everything together for Eqn.~\eqref{eqn:elliptical}, the following holds with probability at least $1 - \frac{\delta}{2}$:
\begin{align}
    {\color{brown}S_T} = (a) + (b) &\leq \underbrace{2 d^2 \log\left(1 + \frac{\kappa T}{d^2} \right)}_{(a) \leq} + \underbrace{\frac{(\kappa \wedge 1)^2}{2} {\color{brown}S_T} + \frac{2(e - 2)}{\kappa \wedge 1} \log\frac{2}{\delta}}_{(b) \leq} \nonumber \\
    &\leq \frac{1}{2} {\color{brown}S_T} + 2 d^2 \log\left(1 + \frac{\kappa T}{d^2} \right) + \frac{2(e - 2)}{\kappa \wedge 1} \log\frac{2}{\delta}, \label{eqn:self-bounding-regret}
\end{align}
which is a \emph{\textbf{\color{teal}linear, self-bounding inequality}}!

Solving for ${\color{brown}S_T}$, we have that with probability at least $1 - \frac{\delta}{2}$,
\begin{equation}
    {\color{brown}S_T} = \sum_{t=1}^T \E_{t-1}\left[ \vv_t^\top \mV_t^{-1} \vv_t \right] \leq 4 d^2 \log\left(1 + \frac{\kappa T}{d^2} \right) + \frac{4(e - 2)}{\kappa \wedge 1} \log\frac{2}{\delta}.
\end{equation}

Combining everything at Eqn.~\eqref{eqn:regret-gs-1}, we have that with probability at least $1 - \delta$ (after union bound with the confidence sequence),
\begin{equation}
    \MBRReg_{\eta}(T) \lesssim \eta \beta \kappa^{-1} d^2 C_{\min}^{-1} \left( \log\frac{T}{d} \right) \left( d^2 \log\frac{\kappa T}{d} + \kappa^{-1} \log\frac{1}{\delta} \right).
\end{equation}

\paragraph{Part II. $\tilde{\gO}(\sqrt{T})$ Regret Bound.}
By the $\eta$-independent bound of \cref{thm:regularized}, we get
\begin{align}
    \MBRReg_{\eta}(T) &\leq L_\mu \sum_{t=1}^T \E_{\bm\phi_t \sim \hat{\pi}_t}\left[ \bignorm{\mE_t \bm\phi_t}_2 + \bignorm{\mE_t \bm\phi_t}_2^2 \right] \\
    &\leq L_\mu \sqrt{T \sum_{t=1}^T \E_{\bm\phi_t \sim \hat{\pi}_t}\left[ \bignorm{\mE_t \bm\phi_t}_2^2 \right]} + L_\mu \sum_{t=1}^T \E_{\bm\phi_t \sim \hat{\pi}_t}\left[ \bignorm{\mE_t \bm\phi_t}_2^2 \right] \tag{Cauchy-Schwarz \& Jensen} \\
    &\lesssim \sqrt{T \kappa^{-1} d^4 C_{\min}^{-1} \left( \log\frac{T}{d} \right)^2} + \kappa^{-1} d^4 C_{\min}^{-1} \left( \log\frac{T}{d} \right)^2 \tag{Part I} \\
    &= \kappa^{-\frac{1}{2}} C_{\min}^{-\frac{1}{2}} d^2 \sqrt{T} \log\frac{T}{d} + \kappa^{-1} C_{\min}^{-1} d^4 \left( \log\frac{T}{d} \right)^2.
\end{align}
\qed

\subsection{\texorpdfstring{Proof of \cref{lem:coverage}: Coverage Lemma}{Proof of Lemma B.2: Coverage Lemma}}
We begin with this simple yet effective matrix lemma that will be useful throughout (we provide its proof at \cref{app:psd-proof} for completeness):
\begin{lemma}
\label{lem:psd}
    If $\mA \succeq \vzero$ and $\mB \succeq \mC \succeq \vzero$ of compatible sizes, then $\mA \otimes \mB \succeq \mA \otimes \mC$ and $\tr(\mA \mB) \geq \tr(\mA \mC)$.
\end{lemma}

We begin from the expectation on the RHS and work our way to the LHS:
\begin{align}
    \E_{\tilde{\bm\phi} \sim \rho} \left[ (\bm\phi \otimes \tilde{\bm\phi})^\top \mM (\bm\phi \otimes \tilde{\bm\phi}) \right] &= \E_{\tilde{\bm\phi} \sim \rho} \left[ \tr\left( (\bm\phi \otimes \tilde{\bm\phi})^\top \mM (\bm\phi \otimes \tilde{\bm\phi}) \right) \right] \\
    &= \E_{\tilde{\bm\phi} \sim \rho} \left[ \tr\left( \mM (\bm\phi \otimes \tilde{\bm\phi}) (\bm\phi \otimes \tilde{\bm\phi})^\top \right) \right] \tag{Cyclic property of $\tr(\cdot)$} \\
    &= \tr\left( \mM\E_{\tilde{\bm\phi} \sim \rho} \left[ (\bm\phi \otimes \tilde{\bm\phi}) (\bm\phi \otimes \tilde{\bm\phi})^\top \right] \right) \tag{Linearity of expectation} \\
    &= \tr\left( \mM \E_{\tilde{\bm\phi} \sim \rho} \left[ \bm\phi \bm\phi^\top \otimes \tilde{\bm\phi} \tilde{\bm\phi}^\top \right] \right). \tag{Mixed-product property of $\otimes$}
\end{align}
Now, let us examine $\E_{\tilde{\bm\phi} \sim \rho} \left[ \bm\phi \bm\phi^\top \otimes \tilde{\bm\phi} \tilde{\bm\phi}^\top \right]$.
Due to the linearity of the expectation and \Cref{ass:REC}, we have
\begin{equation}
    \E_{\tilde{\bm\phi} \sim \rho} \left[ \bm\phi \bm\phi^\top \otimes \tilde{\bm\phi} \tilde{\bm\phi}^\top \right]
    = \bm\phi \bm\phi^\top \otimes \E_{\tilde{\bm\phi} \sim \rho} \left[ \tilde{\bm\phi} \tilde{\bm\phi}^\top \right]
    \succeq C_{\min} \bm\phi \bm\phi^\top \otimes \mI_d,
\end{equation}
where the last Lowener order follows from \cref{lem:psd}.
Then, again by \cref{lem:psd}, we have that
\begin{align}
    \E_{\tilde{\bm\phi} \sim \rho} \left[ (\bm\phi \otimes \tilde{\bm\phi})^\top \mM (\bm\phi \otimes \tilde{\bm\phi}) \right] &= \tr\left( \mM \bm\phi \bm\phi^\top \otimes \E_{\tilde{\bm\phi} \sim \rho} \left[ \tilde{\bm\phi} \tilde{\bm\phi}^\top \right] \right) \\
    &\geq C_{\min} \tr\left( \mM \bm\phi \bm\phi^\top \otimes \mI_d \right) \\
    &= C_{\min} \tr\left( \mM \bm\phi \bm\phi^\top \otimes \left( \sum_{j=1}^d \ve_j\ve_j^\top \right) \right) \\
    &= C_{\min} \sum_{j=1}^d \tr\left( \mM \bm\phi \bm\phi^\top \otimes \ve_j \ve_j^\top \right) \\
    &= C_{\min} \sum_{j=1}^d \tr\left( \mM (\bm\phi \otimes \ve_j) (\bm\phi \otimes \ve_j)^\top \right) \tag{Mixed-product property of $\otimes$} \\
    &= C_{\min} \sum_{j=1}^d (\bm\phi \otimes \ve_j)^\top \mM (\bm\phi \otimes \ve_j).
\end{align}
\qed

\subsection{\texorpdfstring{Proof of \cref{lem:psd}: PSD Matrix Lemma}{Proof of Lemma B.4: PSD Matrix Lemma}}
\label{app:psd-proof}

Let $\mD := \mB - \mC$. Since $\mB \succeq \mC$, we have $\mD \succeq \vzero$.

We first show that $\mA \otimes \mB \succeq \mA \otimes \mC$. By bilinearity of the Kronecker product,
\begin{equation}
\mA \otimes \mB - \mA \otimes \mC
= \mA \otimes (\mB - \mC)
= \mA \otimes \mD.
\end{equation}
It therefore suffices to prove that $\mA \otimes \mD$ is positive semidefinite.

Let $\{\lambda_i(\mA)\}_{i=1}^n$ and $\{\lambda_j(\mD)\}_{j=1}^m$ denote the eigenvalues of $\mA$ and $\mD$, respectively. A standard property of Kronecker products implies that the eigenvalues of $\mA \otimes \mD$ are
\[
\bigl\{\lambda_i(\mA)\lambda_j(\mD)\bigr\}_{i \in [n],\, j \in [m]}.
\]
Since $\mA \succeq \vzero$ and $\mD \succeq \vzero$, all eigenvalues $\lambda_i(\mA)$ and $\lambda_j(\mD)$ are nonnegative, hence so are all products $\lambda_i(\mA)\lambda_j(\mD)$. Therefore $\mA \otimes \mD \succeq \vzero$, which yields $\mA \otimes \mB \succeq \mA \otimes \mC$.

Now we prove $\tr(\mA \mB) \geq \tr(\mA \mC)$. Since $\mA \succeq 0$, there exists a symmetric square root $\mA^{1/2}$ such that
$\mA = \mA^{1/2}\mA^{1/2}$. Using cyclicity of trace,
\begin{equation}
\tr(\mA\mD)
= \tr(\mA^{1/2}\mA^{1/2}\mD)
= \tr(\mA^{1/2}\mD\mA^{1/2}).
\end{equation}
Moreover, $\mA^{1/2}\mD\mA^{1/2}$ is positive semidefinite. For any $\vx \in \mathbb{R}^n$,
letting $\vy := \mA^{1/2}\vx$, we have
\begin{equation}
\vx^\top \mA^{1/2}\mD\mA^{1/2}\vx
= \vy^\top \mD \vy
\ge 0,
\end{equation}
where the inequality uses $\mD \succeq \vzero$. Hence $\mA^{1/2}\mD\mA^{1/2} \succeq \vzero$, and therefore
$\tr(\mA^{1/2}\mD\mA^{1/2}) \ge 0$. This implies $\tr(\mA\mD)\ge 0$, completing the proof.
\qed

\newpage
\section{\texorpdfstring{Proof of \cref{thm:regularized-sqrt}: Regret Bound of \texttt{Explore-Then-Commit}}{Proof of Theorem 5.2: Regret Bound of Explore-Then-Commit}}
\label{app:etc}

Recall the nuclear-norm regularized MLE~\citep{fan2019generalized,lee2025gl-lowpopart}: denoting $\mX_t := \phi(\vx_t, \va_t^1) \phi(\vx_t, \va_t^2)^\top$,
\begin{align}
    \widehat{\bm\Theta}_{T_0} &\gets \argmin_{\bm\Theta \in \Skew(d)} \gL_{T_0}(\bm\Theta)
    + \lambda_{T_0} \bignorm{\bm\Theta}_\nuc, \label{eqn:nuc-estimator-app} \\
    \gL_{T_0}(\bm\Theta) &:= \frac{1}{T_0} \sum_{t=1}^{T_0} \left\{ m(\langle \bm\Theta, \mX_t \rangle) - r_t \langle \bm\Theta, \mX_t \rangle \right\}.
\end{align}

We also introduce the following assumption used commonly in logistic and generalized linear bandits~\citep{abeille2021logistic,russac2021glm,lee2024glm}:
\begin{assumption}
\label{assumption:self-concordance}
	The link function $\mu$ is \textbf{self-concordant with constant $R_s \geq 0$}, i.e.,
	\begin{equation}
		\left| \ddmu\left( \bm\phi^\top \bm\Theta \bm\phi' \right) \right| \leq R_s \dmu\left( \bm\phi^\top \bm\Theta \bm\phi' \right), \quad \forall \bm\phi, \bm\phi' \in \gB^d(1), \forall \bm\Theta \in \Skew(d, 2r; S).
	\end{equation}
\end{assumption}
Lastly, we define the following instance-specific curvature quantity:
\begin{equation}
	\kappa_\star := \min_{\bm\phi, \bm\phi' \in \mathcal{B}^d(1)} \dot{\mu}\left( \bm\phi^\top \bm\Theta_\star \bm\phi' \right).
\end{equation}
This parameter has been identified as a fundamental instance-dependent factor in the analysis of logistic and generalized linear bandits \citep{abeille2021logistic,lee2024glm}.
In contrast to the global curvature $\kappa$ (see \cref{def:gbpm}), which involves an additional minimization over $\bm\Theta_\star \in \mathrm{Skew}(d, 2r; S)$, $\kappa_\star$ captures the local geometry around the true parameter $\bm\Theta_\star$.
We note that in many practical regimes, the global bound may be overly pessimistic, such that $\kappa^{-1} \gg \kappa_\star^{-1}$.
In this Appendix, we show that under the additional assumption that $\mu$ is self-concordant, we can obtain dependencies w.r.t. $\kappa_\star^{-1}$ instead of $\kappa^{-1}$.

The following lemma provides a Frobenius error guarantee for the nuclear-norm regularized MLE $\widehat{\bm\Theta}$: 
\begin{lemma}[Theorem 3.2 of \citet{lee2025gl-lowpopart}]
\label{lem:estimator}
    Let $\delta \in (0, 1)$, and $\tilde{\kappa} \in \{\kappa, \kappa_\star\}$ is such that $\tilde{\kappa} = \kappa_\star$ if \Cref{assumption:self-concordance} holds, and $\tilde{\kappa} = \kappa$ if otherwise.
    Suppose that ${\color{violet} T_0 \gtrsim \tilde{\kappa}^{-2} C_{\min}^{-4} d r \log\frac{d}{\delta}}$.
    Then, with $\lambda_{T_0} = \sqrt{\frac{32 L_\mu \log\frac{4d}{\delta}}{T_0}}$, we have the following:
    \begin{equation}
        \sP\left( \bignorm{\widehat{\bm\Theta} - \bm\Theta_\star}_F \lesssim \frac{1}{\tilde{\kappa} C_{\min}^2} \sqrt{\frac{L_\mu r \log\frac{d}{\delta}}{T_0}} \right) \geq 1 - \delta.
    \end{equation}
\end{lemma}

We first invoke the $\eta$-dependent bound of \cref{thm:regularized} and the above lemma to $t \in [[T_0 + 1, T]]$:
\begin{equation}
    \MBRReg_{\eta}(T) \lesssim T_0 + \eta \beta \sum_{t=T_0+1}^T \E_{\bm\phi_t \sim \hat{\pi}_t}\left[ \bignorm{\mE_t \bm\phi_t}_2^2 \right]
    \leq T_0 + T \eta \beta \bignormop{\mE_{T_0}}^2
    \lesssim T_0 + \frac{T \eta \beta}{\tilde{\kappa}^2 C_{\min}^4} \frac{r \log\frac{d}{\delta}}{T_0}.
\end{equation}
This balances out when $T_0 \asymp \tilde{\kappa}^{-1} C_{\min}^{-2} \sqrt{T \eta \beta r \log\frac{d}{\delta}}$.

Now, we invoke the $\eta$-independent bound of \cref{thm:regularized}, which yields
\begin{align}
    \MBRReg_{\eta}(T) &\lesssim T_0 + \sum_{t=T_0+1}^T \E_{\bm\phi_t \sim \hat{\pi}_t}\left[ \bignorm{\mE_t \bm\phi_t}_2 \right] + \sum_{t=T_0+1}^T \E_{\bm\phi_t \sim \hat{\pi}_t}\left[ \bignorm{\mE_t \bm\phi_t}_2^2 \right] \\
    &\leq T_0 + T \bignormop{\mE_{T_0}} + T \bignormop{\mE_{T_0}}^2 \\
    &\lesssim T_0 + \frac{T}{\tilde{\kappa} C_{\min}^2} \sqrt{\frac{L_\mu r \log\frac{d}{\delta}}{T_0}} + \frac{T}{\tilde{\kappa}^2 C_{\min}^4} \frac{L_\mu r \log\frac{d}{\delta}}{T_0}.
\end{align}
This balances out when $T_0 \asymp \left( T^2 \tilde{\kappa}^{-2} C_{\min}^{-4} r \log\frac{d}{\delta} \right)^{1/3}$ and $T \gtrsim \tilde{\kappa} C_{\min}^2 r \log\frac{d}{\delta}$, the latter which we can assume to hold without loss of any generality.
\qed

\begin{remark}[Unknown $T$]
    We assume $T$ is known to optimally tune $T_0$. If $T$ is unknown, the standard doubling trick~\citep{auer1995casino,besson2018doubling} yields the same regret bound up to a constant factor.
\end{remark}

\newpage
\newcommand{\eEdim}{\underline{\mathrm{Edim}}}
\newcommand{\Edim}{\mathrm{Edim}}
\newcommand{\SEC}{\mathrm{SEC}}

\section{Eluder Dimension of GBPM}
\label{app:eluder}

\subsection{\texorpdfstring{Upper Bounding the Eluder Dimension of \citet{wu2025greedy}}{Upper Bounding the Eluder Dimension of Wu et al. (2025)}}
\label{app:eluder-wu-gbpm}

To instantiate the regret bound of \citet{wu2025greedy} in \cref{app:wu}, we first recall their specific notion of the eluder dimension:
\begin{definition}[Eluder dimension, General Preference Model \citep{wu2025greedy}]
\label{def:eluder-wu}
    Under the general preference model, for any $\mathcal{D}_{t-1} = \{(\vx_i, \va^1_i , \va^2_i)\}_{i=1}^{t-1}$, we define the uncertainty of $(\vx, \va^1 , \va^2)$ with respect to $\mathcal{P}$ as
    \begin{equation}
        U_{\texttt{GP}}(\lambda, \vx, \va^1 , \va^2; \mathcal{P}, \mathcal{D}_{t-1}) = \sup_{P_1, P_2 \in \mathcal{P}} \frac{|P_1(\va^1 \succ \va^2 \mid \vx) - P_2(\va^1 \succ \va^2 \mid \vx)|}{\sqrt{\lambda + \sum_{s=1}^{t-1} (P_1(\va^1_s \succ \va^2_s \mid\vx_s ) - P_2(\va^1_s \succ \va^2_s \mid \vx_s ))^2}}.
    \end{equation}
    Then the eluder dimension of $\gP$ is defined as
    \begin{equation}
        d (\mathcal{P}, \lambda, T) := \sup_{\vx_{1:T}, \va^1_{1:T}, \va^2_{1:T}} \sum_{t \in [T]} \min \left\{1, \left[U_{\texttt{GP}}(\lambda, \vx_t, \va^1_t , \va^2_t; \mathcal{P}, \mathcal{D}_{t-1})\right]^2 \right\}.
    \end{equation}
\end{definition}

Using the standard elliptical potential arguments, we now derive an upper bound for this complexity measure under the GBPM:

\begin{proposition}
\label{prop:eluder-wu}
    For GBPM, the eluder dimension $d (\mathcal{P}, \lambda, T)$ is bounded as
    \begin{equation}
        d (\mathcal{P}, \lambda, T) \leq \frac{2d^2 L_\mu ^2}{\kappa ^2}  \log \left( 1 + \frac{4 \kappa^2 S^2 T}{d^2 \lambda} \right).
    \end{equation}
\end{proposition}

\begin{proof}
    For the proof, let us arbitrarily fix a sequence of context-action-action pairs $(\vx_{1:T}, \va^1_{1:T}, \va^2_{1:T})$, and let us denote the induced sequence of features as $\bm\phi_{1:T}^1$ and $\bm\phi_{1:T}^2$, where $\bm\phi_t^i := \bm\phi(\vx_t, \va_t^i)$ for $t \in [T]$ and $i \in \{1, 2\}.$
    Let us also denote $\bm\Phi_t := \bm\phi_t^1 (\bm\phi_t^2)^\top$.
    For any $\bm\Theta \in \Theta$, the induced preference model is defined as
    \begin{equation}
        P_{\bm\Theta}(\bm\phi_t^1, \bm\phi_t^2) := \mu\left((\bm\phi_t^1)^\top \bm\Theta \phi_t^2 \right)
        = \mu\left( \langle \bm\Theta, \bm\Phi_t \right).
    \end{equation}
    
    We bound the uncertainty via the elliptical potential lemma (\Cref{lem:EPL}).
    Let us denote $P_i = P_{\bm\Theta_i}$ for some arbitrary $\bm\Theta_i \in \Theta.$
    We first upper bound the numerator as follows:
    denoting $\alpha(\vx, \vtheta_1, \vtheta_2) \coloneqq \int_0^1 \dot{\mu}\bigl(\vx^\top \vtheta_1 + z\, \vx^\top (\vtheta_2 - \vtheta_1)\bigr)\, dz$ for $\vx, \vtheta_1, \vtheta_2 \in \sR^d$,
    \begin{align}
        \left| \mu\left( \langle \bm\Theta_1, \bm\Phi_t \rangle \right) - \mu\left( \langle \bm\Theta_2, \bm\Phi_t \rangle \right) \right| & \quad = \left| \alpha\!\left( \operatorname{vec}(\bm{\Phi}_t), \operatorname{vec}(\bm\Theta_1), \operatorname{vec}(\bm\Theta_2) \right) \langle \bm\Theta_1 - \bm\Theta_2, \bm\Phi_t \rangle \right| \tag{Mean-value theorem} \\
        & \quad = \left| \left\langle \bm\Theta_1 - \bm\Theta_2, \underbrace{ \alpha\!\left( \operatorname{vec}(\bm{\Phi}_t), \operatorname{vec}(\bm\Theta_1), \operatorname{vec}(\bm\Theta_2) \right) \bm\Phi_t}_{\triangleq \bar{\bm\Phi}_t(\bm\Theta_1, \bm\Theta_2)} \right\rangle \right| \\
        &\quad \leq \bignorm{\vec( \bm\Theta_1 - \bm\Theta_2 )}_{\mG_t(\bm\Theta_1, \bm\Theta_2)} \bignorm{\vec(\bar{\bm\Phi}_t(\bm\Theta_1, \bm\Theta_2))}_{\mG_t(\bm\Theta_1, \bm\Theta_2)^{-1}} \tag{Cauchy-Schwarz inequality}
    \end{align}
    for some matrix $\mG_t(\bm\Theta_1, \bm\Theta_2) \succ \vzero$ to be determined later.
    
    For the denominator squared:
    \begin{align}
        &\lambda + \sum_{s=1}^{t-1} \left( \mu\left( \langle \bm\Theta_1, \bm{\Phi}_s \rangle \right)
            - \mu\left( \langle \bm\Theta_2, \bm{\Phi}_s \rangle \right) \right)^2 \\
        &= \lambda + \sum_{s=1}^{t-1}
        \left[
            \alpha\!\left(
                \operatorname{vec}(\bm{\Phi}_s),
                \operatorname{vec}(\bm\Theta_1),
                \operatorname{vec}(\bm{\Theta}_2)
            \right)^2
            \langle \bm\Theta_1 - \bm{\Theta}_2 , \bm{\Phi}_s \rangle^2
        \right] \\
        &= \lambda + \operatorname{vec}\!\left( \bm\Theta_1 - \bm{\Theta}_2 \right)^\top
        \left[
            \sum_{s=1}^{t-1}
            \left(
               \alpha\!\left(
                \operatorname{vec}(\bm{\Phi}_s),
                \operatorname{vec}(\bm\Theta_1),
                \operatorname{vec}(\bm{\Theta}_2)
            \right)^2
                \operatorname{vec}(\bm{\Phi}_s)\,
                \operatorname{vec}(\bm{\Phi}_s)^\top
            \right)
        \right]
        \operatorname{vec}\!\left(  \bm\Theta_1 - \bm{\Theta}_2 \right) \\
        &\geq \vec( \bm\Theta_1 - \bm{\Theta}_2)^\top \underbrace{\left[ \frac{\lambda}{4S^2} \mI + \sum_{s=1}^{t-1} \!\left[ \alpha\!\left(
                \operatorname{vec}(\bm{\Phi}_s),
                \operatorname{vec}(\bm\Theta_1),
                \operatorname{vec}(\bm{\Theta}_2)
            \right)^2 
            \operatorname{vec}(\bm{\Phi}_s) \operatorname{vec}(\bm{\Phi}_s)^\top \right] \right]}_{\triangleq \mG_t(\bm\Theta_1, \bm\Theta_2)} \vec( \bm\Theta_1 - \bm{\Theta}_2) \\
        &= \bignorm{\vec( \bm\Theta_1 - \bm{\Theta}_2)}_{\mG_t(\bm\Theta_1, \bm\Theta_2)}^2.
    \end{align}

    Combining the above two inequalities, we have that
    \begin{align}
        U_{\texttt{GP}}(\lambda, \vx, \va^1 , \va^2; \mathcal{P}, \mathcal{D}_{t-1}) &\leq \sup_{\bm\Theta_1, \bm\Theta_2 \in \Theta} \frac{\bignorm{\vec( \bm\Theta_1 - \bm\Theta_2 )}_{\mG_t} \bignorm{\vec(\bar{\bm\Phi}_t(\bm\Theta_1, \bm\Theta_2))}_{\mG_t(\bm\Theta_1, \bm\Theta_2)^{-1}}}{\bignorm{\vec( \bm\Theta_1 - \bm{\Theta}_2)}_{\mG_t(\bm\Theta_1, \bm\Theta_2)}} \\
        &= \sup_{\bm\Theta_1, \bm\Theta_2 \in \Theta} \bignorm{\vec(\bar{\bm\Phi}_t(\bm\Theta_1, \bm\Theta_2))}_{\mG_t(\bm\Theta_1, \bm\Theta_2)^{-1}}
    \end{align}
    
    We can lower-bound $\mG_t(\bm\Theta_1, \bm\Theta_2)$ in the L\"owner sense using the fact that $\alpha(\cdot, \cdot, \cdot) \geq \kappa$,
    \begin{equation}
        \mG_t(\bm\Theta_1, \bm\Theta_2)
        \succeq \frac{\lambda}{4S^2} \mI
        + \sum_{s=1}^{t-1} \kappa^2 \operatorname{vec}({\bm{\Phi}}_s) \operatorname{vec}({\bm{\Phi}}_s)^\top
        = \frac{\lambda}{4S^2} \mI
        + \sum_{s=1}^{t-1} \operatorname{vec}(\kappa {\bm{\Phi}}_s) \operatorname{vec}(\kappa {\bm{\Phi}}_s)^\top \triangleq \mV_t.
    \end{equation}
    Then, as $\alpha(\cdot, \cdot, \cdot) \leq L_\mu$, we have that
    \begin{equation}
        U_{\texttt{GP}}(\lambda, \vx, \va^1 , \va^2; \mathcal{P}, \mathcal{D}_{t-1}) \leq \frac{L_{\mu}}{\kappa} \bignorm{\operatorname{vec}(\kappa \bm\Phi_t)}_{\mV_t^{-1}}.
    \end{equation}
    We then conclude the proof via the elliptical potential lemma (\cref{lem:EPL}):
    \begin{align}
        d(\gP, \lambda, T) &\leq \sum_{t=1}^T \min\left\{1,  \frac{L_{\mu}^2}{\kappa^2} \bignorm{\operatorname{vec}(\kappa \bm\Phi_t)}_{\mV_t^{-1}}^2 \right\} \\
        & \leq \frac{L_{\mu}^2}{\kappa^2} \sum_{t=1}^T \min\left\{1, \bignorm{\operatorname{vec}(\kappa \bm\Phi_t)}_{\mV_t^{-1}}^2 \right\} \tag{$\kappa \leq L_\mu$} \\
        & \leq \frac{2 d^2 L_{\mu}^2}{\kappa^2} \log\left( 1 + \frac{ 4 \kappa^2 S^2 T}{d^2 \lambda} \right).
    \end{align}
\end{proof}

\subsection{Connection to the Standard Eluder Dimensions}
\label{app:eluder-standard}

In this appendix, we will elucidate the connection between \cref{def:eluder-wu} and two other ``standard'' definitions of eluder-type complexities: \emph{sequential extrapolation coefficient (SEC)}~\citep{xie2023coverage} and the original eluder dimension~\citep{russo-vanroy}
For simplicity, we denote $\gZ := \gX \times \gA \times \gA$, and $P(\vz) := P(\va^1 \succ \va^2 \mid \vx)$ for $\vz = (\vx, \va^1, \va^2)$.

First, we recall the definition of SEC adopted for our setting:
\begin{definition}[$\lambda$-regularized SEC, Definition 7 of \citet{xie2023coverage}]
\label{def:sec-lambda}
    Let $\Pi \subseteq \Delta(\gZ)$ be a distribution class. Then, the SEC is defined as
    \begin{equation}
    \SEC_{\lambda}(\gP, \Pi, T)
    := \sup_{P_{1:T}^1, P_{1:T}^2 \subseteq \gP} \ \sup_{\nu_{1:T} \subseteq \Pi}  
    \sum_{t=1}^T
    \frac{\left(\E_{\vz_t \sim \nu_t}[P_t^1(\vz_t)- P_t^2(\vz_t)]\right)^2}{
    \lambda + \sum_{s=1}^{t-1} \E_{\vz_s \sim \nu_s}\big[(P_t^1(\vz_s) - P_t^2(\vz_s))^2\big]}.
\end{equation}
\end{definition}

Note that when the distribution class is restricted to the set of Dirac measures $\mD:=\{\delta_{\vz}:\vz\in Z\}$, we have the following relationship between SEC and \cref{def:eluder-wu}:
\begin{align}
    d(\gP, \lambda, T) &= \sup_{\vz_{1:T}} \sum_{t \in [T]} \min \left\{1, \sup_{P_1, P_2 \in \mathcal{P}} \frac{(P_1(\vz_t) - P_2(\vz_t))^2}{\lambda + \sum_{s=1}^{t-1} (P_1(\vz_s) - P_2(\vz_s))^2} \right\} \\
    &\leq \sup_{\vz_{1:T}} \sum_{t \in [T]} \sup_{P_1, P_2 \in \mathcal{P}} \frac{(P_1(\vz_t) - P_2(\vz_t))^2}{\lambda + \sum_{s=1}^{t-1} (P_1(\vz_s) - P_2(\vz_s))^2} \\
    &= \sup_{P_{1:T}^1, P_{1:T}^2} \sup_{\vz_{1:T}} \sum_{t \in [T]} \frac{(P_t^1(\vz_t) - P_t^2(\vz_t))^2}{\lambda + \sum_{s=1}^{t-1} (P_s^1(\vz_s) - P_s^2(\vz_s))^2} \\
    &= \sup_{P_{1:T}^1, P_{1:T}^2} \sup_{\nu_{1:T} \subseteq \mD} \sum_{t \in [T]} \frac{(\E_{\vz_t \sim \nu_t}[P_t^1(\vz_t) - P_t^2(\vz_t)])^2}{\lambda + \sum_{s=1}^{t-1} (\E_{\vz_s \sim \nu_s}[P_t^1(\vz_s) - P_t^2(\vz_s)])^2} \\
    &= \SEC_{\lambda}(\gP, \mD, T).
\end{align}

Second, we recall the standard eluder dimension of \citet{foster2021instance} and \citet{li2022eluder}:\footnote{The original definition is due to \citet{russo-vanroy} and slightly different, but as mentioned in \citet{li2022eluder}, the ``new'' definition is ``never larger and is sufficient to analyze all the applications of eluder dimension in literature.''}
\begin{definition}[Definition 1 of \citet{li2022eluder}]
\label{def:eluder}
    For any fixed preference $P^* \in \gP$, and scale $\varepsilon \geq 0$, the \textbf{exact eluder dimension} $\eEdim_{P^*}(\gP, \varepsilon)$ is the largest $m \in \sN$ such that there exists a sequence $\{ (\vz_t, P_t) \}_{t \in [m]} \subset \gZ \times \gP$ such that the following holds:
    for all $t \in [m]$,
    \begin{equation}
        \left| P_t(\vz_t) - P^*(\vz_t) \right| > \varepsilon, \quad \text{and} \quad \sum_{s < t} \left( P_t(\vz_s) - P^*(\vz_s) \right)^2 < \varepsilon^2.
    \end{equation}
    Then for all $\varepsilon > 0$, we define:
    \begin{itemize}
        \item The \textbf{eluder dimension} is $\Edim_{P^*}(\gP, \varepsilon) := \sup_{\varepsilon' \geq \varepsilon} \eEdim_{P^*}(\gP, \varepsilon')$.
        \item $\eEdim(\gP, \varepsilon) := \sup_{P^* \in \gP} \eEdim_{P^*}(\gP, \varepsilon')$ and $\Edim(\gP, \varepsilon) := \sup_{P^* \in \gP} \Edim_{P^*}(\gP, \varepsilon')$.
    \end{itemize}
\end{definition}

We first prove that $\Edim(\gP, \varepsilon)$ and $\SEC_\lambda(\gP,\mD,T)$ are equivalent up to some constants and logarithmic factors:

\begin{proposition}
\label{prop:eluder-sec}
    Suppose that $\gZ \subseteq \gB^d(1)$ and $\Theta \subseteq \Skew(d; 2r, S)$ for some $S > 0$.
    Let $\varepsilon > 0$, $T \geq \Edim(\gP, \varepsilon)$, and $\lambda \geq 1$.
    Then,
    \begin{equation}
         \frac{\varepsilon^2 \Edim(\gP, \varepsilon)}{\lambda + \varepsilon^2}
        \leq \SEC_\lambda(\gP, \mD, T) 
        \lesssim \Edim(\gP, T^{-1/2}) \log T.
    \end{equation}
\end{proposition}

\begin{proof}
    We prove each direction separately.
    
    \textbf{Upper Bound.}
    Noting that for $\lambda \geq 1$, $\SEC_\lambda(\gP, \mD, T) \leq \SEC_1(\gP, \mD, T)$, this immediately follows from \citet[Proposition 7]{xie2023coverage} with $\gD = \mD$.
    
    \textbf{Lower Bound.}
    Consider the eluder witness, i.e., a sequence of $\{(\vz_t, P_t)\}_{t \in {d_e}}$ and some fixed preference $P^*$ that attains the eluder dimension $d_e := \Edim(\gP, \varepsilon)$.
    Then, by definition, 
    \begin{align}
        \SEC_{\lambda}(\gP, \mD, T) &= \sup_{P_{1:T}^1, P_{1:T}^2 \subseteq \gP} \ \sup_{\vz_{1:T} \subseteq \gZ} \sum_{t=1}^T \frac{\left(P_t^1(\vz_t)- P_t^2(\vz_t)\right)^2}{
    \lambda + \sum_{s=1}^{t-1} (P_t^1(\vz_s) - P_t^2(\vz_s))^2} \\
        &\geq \sup_{P_{1:T}^1 \subseteq \gP} \ \sup_{\vz_{1:T} \subseteq \gZ} \sum_{t=1}^T
    \frac{\left(P_t^1(\vz_t)- P^*(\vz_t)\right)^2}{
    \lambda + \sum_{s=1}^{t-1} (P_t^1(\vz_s) - P^*(\vz_s))^2} \tag{Set $P_t^2 = P^*$ for all $t \in [T]$} \\
        &\geq \sum_{t=1}^{d_e}
    \frac{\left(P_t(\vz_t)- P^*(\vz_t)\right)^2}{
    \lambda + \sum_{s=1}^{t-1} (P_t(\vz_s) - P^*(\vz_s))^2} \tag{Set $\vz_{1:T}$ and $P^1_{1:T}$ to be the eluder witness sequence} \\
        &> \sum_{t=1}^{d_e} \frac{\varepsilon^2}{\lambda + \varepsilon^2}
        = \frac{d_e \varepsilon^2}{\lambda + \varepsilon^2}.
    \end{align}
\end{proof}

We conclude with a nearly-tight characterization of the eluder dimension of GBPM, whose proof is deferred to the next subsection:

\begin{proposition}
\label{prop:eluder-formal}
    Let $\gX \subseteq \gB^d(1)$ and $\Theta \subseteq \Skew(d, 2r; S)$ for some $S > 0$.
    Define the function class as the \textbf{GBPM}:
    \begin{equation}
        \gP := \left\{ (\vx, \vy) \mapsto \mu\left( \vx^\top \bm\Theta \vy \right) : \bm\Theta \in \Theta \right\},
    \end{equation}
    where we have the following properties: $\kappa \leq \dmu(\vx^\top \bm\Theta \vy) \leq L_\mu$ and $\left| \vx^\top \bm\Theta \vy \right| \leq S$\footnote{This follows from matrix H\"{o}lder inequality: $\left| \vx^\top \bm\Theta \vy \right| = \left| \langle \bm\Theta, \vx \vy^\top \rangle \right| \leq \bignorm{\bm\Theta}_\nuc \bignormop{\vx \vy^\top} \leq S$.}  for all $(\vx, \vy, \bm\Theta) \in \gX \times \gX \times \Theta$.
    Then, its eluder dimension is bounded as follows: for any $\varepsilon < S L_\mu$,
    \begin{equation}
        \Edim(\gP, \varepsilon) \leq \frac{3e}{e - 1} \cdot d^2 \cdot \frac{L_\mu^2}{\kappa^2} \cdot \log\frac{24 S^2 L_\mu^2}{\varepsilon^2}.
    \end{equation}
    Furthermore, if $\mu(z) = \frac{1}{2} + z$, then we have the following lower bound:
    \begin{equation}
        \Edim(\gP, \varepsilon) \geq \binom{d}{2} \log_4 \frac{S}{\sqrt{3} \varepsilon}.
    \end{equation}
\end{proposition}

Note that the same $\widetilde{\Omega}(d^2)$ lower bound applies to the SEC due to \cref{prop:eluder-sec}.
This $\widetilde{\Omega}(d^2)$ scaling implies that eluder-dimension-based frameworks cannot efficiently exploit the low-rank structure of $\bm\Theta$.
The high eluder dimension arises because $\Skew(d;2r)$ contains rank-$2$ ``coordinate spikes'' of the form $(\ve_i\ve_j^\top-\ve_j\ve_i^\top)$.
An adversary can query specific pairs to isolate these directions one-by-one. 
Since the eluder dimension measures worst-case separability rather than metric entropy (which scales as $\gO(dr)$), it reflects the ambient basis size even when the parameter manifold is low-dimensional.

This limitation is best understood through the ``global embedding'' characterization.
Specifically, a standard sufficient condition for bounding the eluder dimension by the $\mu$-rank relies on constructing \emph{global} maps $\phi:\gX\times\gX\to\gB^{d_e}(1)$ and $w:\Theta\to\gB^{d_e}(R)$ such that $\vx^\top\bm\Theta\vy=\langle \phi(\vx,\vy), w(\bm\Theta)\rangle$~\citep[Proposition 4]{li2022eluder}.
While skew-symmetry admits a Schur decomposition $\bm\Theta=\mQ\bm\Lambda\mQ^\top$ that allows the representation
$$
\vx^\top\bm\Theta\vy
=\Big\langle \mathrm{vec}\big((\mQ^\top\vx)(\mQ^\top\vy)^\top\big),\,\mathrm{vec}(\bm\Lambda)\Big\rangle,
$$
this does not yield a valid low-dimensional witness.
Crucially, the feature map depends on the basis $\mQ$ (and thus on the specific parameter $\bm\Theta$), which violates the condition of having a single global feature map across the entire hypothesis class.
Consequently, guarantees relying on such complexity measures (e.g., \texttt{GS} by \citet{wu2025greedy}) incur the full $d^2$ complexity, mirroring the statistical hardness of quadratic functions with full-rank Hessians~\citep[Proposition 3]{osband-vanroy}.

\section{\texorpdfstring{Proof of \Cref{prop:eluder-formal}}{Proof of Proposition F.6}}
For the proof, we recall the notion of \emph{generalized rank} and a useful proposition linking the above two concepts:

\begin{definition}[Definition 3 of \citet{li2022eluder}]
    For a given $\mu : \sR \rightarrow \sR$, the \textbf{$\mu$-rank} of $\gP$ at scale $R > 0$, denoted as $\mu\text{-}\mathrm{rk}(\gP, R)$, is the smallest dimension $d \in \sN$ for which there exist $R_\phi, R_w > 0$ with $R_\phi R_w = R$, and (global) mappings $\phi : \gX \times \gX \rightarrow \gB^d(R_\phi)$ and $w : \gP \rightarrow \gB^d(R_w)$ such that
    \begin{equation}
        P(\vx \succ \vy) = \mu\left( \langle \phi(\vx, \vy), w(P) \rangle \right), \quad \forall (\vx, \vy, P) \in \gX \times \gX \times \gP,
    \end{equation}
    or $\infty$ if no such $d$ exists.
\end{definition}

\begin{proposition}[Proposition 4(ii) of \citet{li2022eluder}]
\label{prop:eluder-rank}
    For all $\varepsilon < R L_\mu$,
    \begin{equation}
        \eEdim(\gP, \varepsilon) \leq \frac{3e}{e - 1} \cdot \mu\text{-}\mathrm{rk}(\gP) \cdot \frac{L_\mu^2}{\kappa^2} \cdot \log\frac{24 R^2 L_\mu^2}{\varepsilon^2}.
    \end{equation}
\end{proposition}

We prove the upper and lower bounds separately.

\textbf{Upper Bound.}
This follows trivially from adapting \cref{prop:eluder-rank} to our setting by considering $\phi : (\vx, \vy) \mapsto \mathrm{vec}(\vx \vy^\top)$ and $w : \bm\Theta \mapsto \mathrm{vec}(\bm\Theta)$.

\textbf{Lower Bound.}
The construction is largely inspired by that of \citet[Proposition 5]{li2022eluder}, which we adapt to our setting.

We will construct a sequence $\{ (\vx_t, \vy_t, \bm\Theta_t) \}_{t \in [m]}$ that witnesses the claimed lower bound with $\bm\Theta_\star = 0$.
The key observation is that $\Skew(d)$ admits the following orthonormal basis: $\gB \triangleq \left\{ \frac{1}{\sqrt{2}} (\ve_i \ve_j^\top - \ve_j \ve_i^\top) \right\}_{1 \leq i < j \leq d}$.

For given $\varepsilon$, let $\alpha \in (\varepsilon, \sqrt{3}\varepsilon)$ and $k := \floor{\log_4 \frac{S}{\alpha}}$.
Then, we can first consider the following sequence of length $k + 1$: for $t \in \{0\} \cup [k]$,
\begin{equation}
    \vx_t = 2^{t - k} \ve_1, \ \vy_t = 2^{t - k} \ve_2, \quad \bm\Theta_t = \alpha \cdot 2^{2(k - t)} (\ve_1 \ve_2^\top - \ve_2 \ve_1^\top).
\end{equation}
For each $t$, we have that
\begin{equation}
    \mu(\vx_t^\top \bm\Theta_t \vy_t) - \mu(0) = \vx_t^\top \bm\Theta_t \vy_t = \alpha > \varepsilon,
\end{equation}
and
\begin{equation}
    \sum_{s < t} \left( \mu(\vx_s^\top \bm\Theta_t \vy_s) - \mu(0) \right)^2 = \sum_{s < t} \left( \vx_s^\top \bm\Theta_t \vy_s \right)^2 = \alpha^2 \sum_{s < t} 4^{2s - 2t} < \frac{1}{15} \alpha^2 < \varepsilon^2.
\end{equation}
As $\vx_t, \vx_t \in \gB^d(1)$ and $\bignorm{\bm\Theta_t}_\nuc \leq \alpha 2^{2k} \leq \alpha 4^{\log_4\frac{S}{\alpha}} = S$, we have $\eEdim(\gP, \varepsilon) \geq k + 1 \geq \log_4 \frac{S}{\alpha} \geq \log_4 \frac{S}{\sqrt{3} \varepsilon}$.

Now we concatenate $\binom{d}{2} = \frac{d (d - 1)}{2}$ times across the basis $\gB$, i.e., 
\begin{equation}
    \vx_{t,i,j} = 2^{t - k} \ve_i, \ \vy_{t,i,j} = 2^{t - k} \ve_j, \quad \bm\Theta_{t,i,j} = \alpha \cdot 2^{2(k - t)} (\ve_i \ve_j^\top - \ve_j \ve_i^\top)
\end{equation}
for $1 \leq i < j \leq d$, and we are done.
\qed

\newpage
\section{\texorpdfstring{Instantiating Regret Bound of \citet{wu2025greedy} to GBPM}{Instantiating Regret Bound of Wu et al. (2025a) to GBPM}}
\label{app:wu}

\subsection{Regret Bound of \texttt{Greedy Sampling}}

For this section, we will consider the reverse KL-regularization as in \citet{wu2025greedy}, i.e., $\psi(\pi)=\KL(\pi, \pi_{\mathrm{ref}})$ for some fixed $\pi_{\mathrm{ref}} \in \Pi$.
Recall that we defined the regularized and unregularized Max-Best-Response Regrets as
\begin{equation}
    \MBRReg_{\eta}(T) := \sum_{t=1}^T \max_{\pi \in \Pi} \left\{ \frac{1}{2} - J_{\eta}(\hat{\pi}_t^1, \pi) \right\}, \quad
    \MBRReg(T) := \sum_{t=1}^T \max_{\pi \in \Pi} \left\{ \frac{1}{2} - J(\hat{\pi}_t^1, \pi) \right\}.
\end{equation}

In this section, we derive the regularized and unregularized regret bound of \texttt{Greedy Sampling (GS)} of \citet{wu2025greedy} for GBPM, based on general function approximation. We first recall the regret bound from \citet{wu2025greedy}:
\begin{theorem}[Theorem 1 of \citet{wu2025greedy}]
    Suppose that the preference class $\gP$ is finite with cardinality $N_\gP = |\gP| < \infty.$
    For any $\delta \in (0, 1)$, with probability at least $1 - \delta$, \texttt{GS} attains the following regret bound:
    \begin{equation}
        \MBRReg_{\eta} (T) = O \left(e^{\eta} \, d(\mathcal{P}, \lambda, T) \log (N_{\mathcal{P}}T / \delta) \right).
    \end{equation}
\end{theorem}

\paragraph{Instantiation for GBPM.}
We first instantiate the KL-regularized regret bound for GBPM:

\begin{theorem}[KL-regularized Regret Bound of \texttt{Greedy Sampling}]
\label{thm:wu-kl-regret}
    For any $\delta \in (0,1)$, with probability with at least $1 - \delta$, \texttt{GS} (when applied to \GBPM) attains the following regret bound:
    \begin{equation}
        \MBRReg_{\eta} (T) \lesssim e^{\eta} \cdot \frac{d^2 L_{\mu}^2}{\kappa^2} \cdot \log T \cdot \left(\log\frac{T}{\delta} + dr \log (LST) \right).
    \end{equation}
\end{theorem}

\begin{proof}
The proof consists of two parts: 1) Extending the preference model class size term for the infinite preference space, and 2) Bounding the eluder dimension of \GBPM.

We first extend the term $N_{\mathcal{P}}=|\gP|$ for the infinite space case $\Theta \triangleq \Skew(d, 2r; S)$ by using covering number arguments.
We denote $P_{\bm{\Theta}}$ as the preference probability given by \GBPM~for parameter $\bm{\Theta}$.
For simplicity, we use the following notations introduced in the previous section. We denote $\gZ := \gX \times \gA \times \gA$ and $P(\vz) := P(\va^1 \succ \va^2 \mid \vx)$  for $\vz = (\vx, \va^1, \va^2)$.
Using these, we have the following lemma, whose proof is provided in \cref{app:prior-mle}:
\begin{lemma}
\label{lem:prior-mle}
    Let $\{(\vz_i, r_i)\}_{i \in [t]}$ be a potentially adaptively collected data with $r_i \sim \Ber(P(\vz_i))$. 
    Denote the (constrained) MLE as $\widehat{\bm{\Theta}}_t := \argmax_{\bm{\Theta} \in \Skew(d, 2r; S)} \sum_{i \in [t]} \ell_i(\bm\Theta)$, where $\ell_i(\bm\Theta) := \left( r_i \log P_{\bm{\Theta}}(\vz_i) + (1 - r_i) \log (1 - P_{\bm{\Theta}}(\vz_i)) \right)$.
    Suppose that $\ell_i(\cdot)$ is $L$-Lipschitz w.r.t. the Frobenius norm.
    Then, for any $\delta \in (0, 1)$ the following holds:
\begin{equation}
    \sP\left( \sum_{i=1}^t (P_{\widehat{\bm{\Theta}}_t}(\vz_i) - P_{\bm{\Theta}_\star}(\vz_i))^2 \lesssim \log\frac{T}{\delta} + dr \log LST\right) \geq 1 - \delta, \quad \forall t \in [T].
\end{equation}
\end{lemma}
With this lemma and our eluder dimension bound (\cref{prop:eluder-wu}), the derivation of the regret bound in \citet{wu2025greedy} follows through, with $\log \left( \frac{N_{\mathcal{P}}T}{\delta} \right)$ and $\lambda$ both replaced with $\log\frac{T}{\delta} + dr \log LST $.
\end{proof}

\paragraph{Converting to Unregularized Regret Bound.}
We now convert the KL-regularized regret bound to its unregularized counterpart via the following lemma:
\begin{lemma}
\label{lem:kl-to-unreg}
    Suppose $D_{\mathrm{ref}} := \max_{\pi \in \Pi} \KL (\pi, \pi_{\mathrm{ref}}) < \infty$. Then we have 
    \begin{equation}
        \MBRReg (T) \leq \MBRReg_{\eta} (T) + \eta^{-1} D_{\mathrm{ref}} T.
    \end{equation}
\end{lemma}
\begin{proof}
Recall that the KL-regularized objective is defined as
\begin{equation}
    J_{\eta}(\pi, \pi') = J(\pi, \pi') - \eta^{-1} \KL(\pi, \pi_{\mathrm{ref}}) + \eta^{-1} \KL(\pi', \pi_{\mathrm{ref}})
\end{equation}
Then,
\begin{align}
    &\MBRReg (T)
    = \sum_{t=1}^T \max_{\pi \in \Pi}\left(\frac{1}{2} - J(\hat{\pi}_t^1, \pi) \right)\\
    &= \sum_{t=1}^T \max_{\pi \in \Pi}\left(\frac{1}{2} - \left( J(\hat{\pi}_t^1, \pi) - \eta^{-1} \KL(\hat{\pi}_t^1, \pi_{\mathrm{ref}}) + \eta^{-1} \KL(\pi, \pi_{\mathrm{ref}}) \right) - \eta^{-1} \KL(\hat{\pi}_t^1, \pi_{\mathrm{ref}}) + \eta^{-1} \KL(\pi, \pi_{\mathrm{ref}}) \right) \\
    & \leq \sum_{t=1}^T \max_{\pi \in \Pi}\left(\frac{1}{2} - J_{\eta}(\hat{\pi}_t^1, \pi) \right) + \sum_{t=1}^T \max_{\pi \in \Pi}\left(- \eta^{-1} \KL(\hat{\pi}_t^1, \pi_{\mathrm{ref}}) + \eta^{-1}\KL(\pi, \pi_{\mathrm{ref}}) \right) \\
     & \leq \MBRReg_{\eta} (T) + \sum_{t=1}^T \max_{\pi \in \Pi}\eta^{-1}\KL(\pi, \pi_{\mathrm{ref}}) \\
     & = \MBRReg_{\eta}  (T) + \eta^{-1} D_{\mathrm{ref}} T.
\end{align}
\end{proof}

Putting everything together, we have the following corollary of \cref{thm:wu-kl-regret} for the unregularized regret bound:

\begin{corollary}[Unregularized Regret Bound of \texttt{GS}]
\label{cor:wu-unreg-regret}
    Suppose $D_{\mathrm{ref}} := \max_{\pi \in \Pi} \KL (\pi, \pi_{\mathrm{ref}}) < \infty$. Then with probability at least $1 - \delta$, the unregularized regret of \texttt{GS} under \GBPM~satisfies
    \begin{equation}
        \MBRReg(T) \lesssim e^\eta d^2 \kappa^{-2} L_\mu ^2 \cdot \log T \cdot \left(\log \frac{T}{\delta} + dr \log(LST) \right) + \eta^{-1} D_{\mathrm{ref}} T.
    \end{equation}
    Furthermore, there exists no $\eta$ (even dependent on $T$) such that the RHS is $\gO(T^{1-\gamma})$ for any $\gamma \in (0, 1]$.
\end{corollary}

\begin{proof}
    The regret bound is a direct result from the KL-regularized regret bound in \cref{thm:wu-kl-regret}, converted to unregularized regret via \cref{lem:kl-to-unreg}.
    
    For the second claim, 
    suppose that this is true. Then, for the second term, we require $\eta^{-1} D_{\mathrm{ref}} T = \gO(T^{1-\gamma})$ to hold, which implies $\eta = \Omega(T^\gamma)$. Plugging this into the first term leads to an additive term of $\gO(e^{T^\gamma})$, which is superpolynomial: a contradiction. This concludes the proof.
\end{proof}

\subsection{\texorpdfstring{Proof of \cref{lem:prior-mle}: MLE Estimator Bound}{Proof of Lemma F.3: MLE Estimator Bound}}
\label{app:prior-mle}
We define the probability mass function (pmf) of $r \mid \vz \sim \Ber(P(\vz))$ as
\begin{equation}
    P(r \mid \vz) = P(\vz)^r (1 - P(\vz))^{1-r}, \quad r \in \{0, 1\}.
\end{equation}
Then we have the following lemma, whose proof is deferred to \cref{app:ye-wu-lem1}:
\begin{lemma}
\label{lem:ye-wu-lem1}
    For each $\bm{\Theta} \in \Theta$ and $t \in [T]$, the following holds:
    \begin{equation}
        \sP\left( \sum_{i=1}^t (P_{\bm{\Theta}}(\vz_i) - P_{\bm{\Theta}_\star}(\vz_i))^2 \leq \log\frac{1}{\delta} +  \sum_{i=1}^t \log\frac{P_{\bm{\Theta}_\star}(r_i \mid \vz_i)}{P_{\bm{\Theta}}(r_i \mid \vz_i)} \right) \geq 1 - \delta
    \end{equation}
\end{lemma}
Let $\Theta_\varepsilon$ be an $\varepsilon$-net of ${\Theta}$ in terms of the Frobenius norm.
Then by the union bound, we have:
\begin{equation}
    \sP\left( \sum_{i=1}^t (P_{\bm{\Theta}}(\vz_i) - P_{\bm{\Theta}_\star}(\vz_i))^2 \leq \log\frac{|\Theta_\varepsilon|}{\delta} +  \sum_{i=1}^t \log\frac{P_{\bm{\Theta}_\star} (r_i \mid \vz_i)}{P_{\bm{\Theta}}(r_i \mid \vz_i)}, \quad \forall \bm{\Theta} \in \Theta_\varepsilon \right) \geq 1 - \delta.
\end{equation}

Let $\widehat{\bm{\Theta}}_{\varepsilon, t}$ be the epsilon net element corresponding to $\widehat{\bm{\Theta}}_t$,
i.e., $\lVert \widehat{\bm{\Theta}}_t - \widehat{\bm{\Theta}}_{\varepsilon, t} \rVert_F \leq \varepsilon$. Then,
\begin{equation}
    \sP\left( \sum_{i=1}^t (P_{\widehat{\bm{\Theta}}_{\varepsilon, t}}(\vz_i) - P_{\bm{\Theta}_\star}(\vz_i))^2 \leq \log\frac{|\Theta_\varepsilon|}{\delta} + \sum_{i=1}^t \log\frac{P_{\bm{\Theta}_\star}(r_i \mid \vz_i)}{P_{\widehat{\bm{\Theta}}_{\varepsilon, t}}(r_i \mid \vz_i)} \right) \geq 1 - \delta.
\end{equation}
Using the inequality $(a - b)^2 \geq \frac{1}{2} (a - c)^2 - (b - c)^2$ and the optimality of the MLE, with probability at least $1 - \delta$ the following holds:
\begin{align}
    & \frac{1}{2} \sum_{i=1}^t (P_{\widehat{\bm{\Theta}}_t}(\vz_i) - P_{\bm{\Theta}_\star}(\vz_i))^2 
    - \sum_{i=1}^t (P_{\widehat{\bm{\Theta}}_t}(\vz_i ) - P_{\widehat{\bm{\Theta}}_{\varepsilon, t}}(\vz_i))^2 \\
    & = \log\frac{|\Theta_\varepsilon|}{\delta} +  \sum_{i=1}^t \log\frac{P_{\bm{\Theta}_\star}(r_i \mid \vz_i)}{P_{\widehat{\bm{\Theta}}_t}(r_i \mid \vz_i)} +  \sum_{i=1}^t \log\frac{P_{\widehat{\bm{\Theta}}_t}(r_i \mid \vz_i)}{P_{\widehat{\bm{\Theta}}_{\varepsilon, t}}(r_i \mid \vz_i)} \\
    & \leq \log\frac{|\Theta_\varepsilon|}{\delta} +  \sum_{i=1}^t \log\frac{P_{\widehat{\bm{\Theta}}_t}(r_i \mid \vz_i)}{P_{\widehat{\bm{\Theta}}_{\varepsilon, t}}(r_i \mid \vz_i)}.
\end{align}

Now, note that for any $\bm\Theta$,
\begin{equation}
    \log P_{\bm\Theta}(r_i \mid \vz_i) = r_i \log P_{\bm\Theta}(\vz_i) + (1 - r_i) \log (1 - P_{\bm\Theta}(\vz_i)),
\end{equation}
which is $L$-Lipschitz in $\bm\Theta$ by given.
With this, we can bound the log sum on the right as
\begin{equation}
    \sum_{i=1}^t \log\frac{P_{\widehat{\bm{\Theta}}_t}(r_i \mid \vz_i)}{P_{\widehat{\bm{\Theta}}_{\varepsilon, t}}(r_i \mid \vz_i)}
    \leq  L t \bignorm{\widehat{\bm\Theta}_{\varepsilon, t} - \widehat{\bm\Theta}_t}_F
    \leq L t \varepsilon.
\end{equation}

Since we assumed that $\log P_{\bm{\Theta}}$ is $L$-Lipschitz, it follows that $P_{\bm{\Theta}}$ is also $L$-Lipschitz.\footnote{As the Lipschitz constant is the maximum gradient norm by the Rademacher's theorem, $\bignorm{\nabla_{\bm\Theta} P_{\bm\Theta}} = P_{\bm\Theta} \cdot \bignorm{\nabla_{\bm\Theta} \log P_{\bm\Theta}} \leq L$.} Therefore,
\begin{equation}
    \sum_{i=1}^t (P_{\widehat{\bm{\Theta}}_t}(\vz_i ) - P_{\widehat{\bm{\Theta}}_{\varepsilon, t}}(\vz_i))^2 \leq Lt \varepsilon^2.
\end{equation}

We now bound the cardinality of the $\varepsilon$-net $|\Theta_\varepsilon|$ by bounding the covering number of a slightly larger set:
\begin{lemma}[Lemma 3.1 of \citet{candes2011tight}]
\label{lem:covering-number}
    Let $\Theta(d, 2r; S) \coloneqq \{ \mX \in \sR^{d \times d} \mid \|\mathbf{X}\|_F \leq S, \operatorname{rank}(\mX) \leq 2r \} \supseteq \Skew(d, 2r; S)$. For any $\varepsilon > 0$, there exists an $\varepsilon$-net $\Theta_\varepsilon$ of $\Theta(d, 2r; S)$ w.r.t. $\bignorm{\cdot}_F$ with $|\Theta_\varepsilon| \leq \left( \frac{9S}{\varepsilon} \right)^{2(2d+1)r}.$
\end{lemma}

Putting the bounds together, we have:

\begin{equation}    
    \sP\left( \sum_{i=1}^t (P_{\widehat{\bm{\Theta}}_t}(\vz_i)  - P_{\bm{\Theta}_\star}(\vz_i))^2 \lesssim \log\frac{1}{\delta} +  Lt \varepsilon + Lt \varepsilon^2 + dr \log\frac{S}{\varepsilon} \right) \geq 1 - \delta.
\end{equation}

Choosing $\varepsilon \approx \frac{1}{(Lt)^2}$,
\begin{equation}
    \sP\left( \sum_{i=1}^t (P_{\widehat{\bm{\Theta}}_t}(\vz_i)  - P_{\bm{\Theta}_\star}(\vz_i))^2 \lesssim \log\frac{1}{\delta} + dr \log(LSt) \right) \geq 1 - \delta.
\end{equation}

Setting $\delta = \delta / T$ and taking the union bound over $t$, we have:
\begin{equation}
    \sP\left( \forall t \in [T], \,  \sum_{i=1}^t (P_{\widehat{\bm{\Theta}}_t}(\vz_i)  - P_{\bm{\Theta}_\star}(\vz_i))^2 \lesssim \log\frac{T}{\delta} + dr \log(LST) \right) \geq 1 - \delta.
\end{equation}
which concludes the proof.
\qed

\section{\texorpdfstring{Proof of \cref{lem:ye-wu-lem1}}{Proof of Lemma H.6}}
\label{app:ye-wu-lem1}
    The proof closely follows that of \citet[Lemma 1]{ye2024general} and \citet[Lemma 3]{wu2025greedy}.
    
    For the function $P_{\bm{\Theta}}$ defined by the fixed $\bm{\Theta} \in \Theta$, we first upper bound its logarithmic moment generating function as
    \begin{align}
        & \log \mathbb{E}\exp\!\left(\sum_{i=1}^t \log \frac{P_{\bm{\Theta}}(r_i \mid \vz_i)}{P_{\bm{\Theta}_\star}(r_i \mid \vz_i)}\right)\\
        & = \log \mathbb{E}\exp\!\Biggl(\sum_{i=1}^{t-1} \log \frac{P_{\bm{\Theta}}(r_i \mid \vz_i)}{P_{\bm{\Theta}_\star}(r_i \mid \vz_i)}
        + \log\Bigl(2\,\mathbb{E}_{r_t \mid \vz_t}\sqrt{\frac{P_{\bm{\Theta}}(r_t \mid \vz_t)}{P_{\bm{\Theta}_\star}(r_t \mid \vz_t)}}\Bigr)\Biggr)
        \\
        &= \log \mathbb{E}\exp\!\Biggl(\sum_{i=1}^{t-1} \log \frac{P_{\bm{\Theta}}(r_i \mid \vz_i)}{P_{\bm{\Theta}}(r_i \mid \vz_i)}
        + \log\Bigl(1 - H\,\bigl(P_{\bm{\Theta}}(r_t \mid \vz_t)\,\|\,P_{\bm{\Theta}_\star}(r_t \mid \vz_t)\bigr)^2\Bigr)\Biggr)
        \\
        &\le \log \mathbb{E}\exp\!\Biggl(\sum_{i=1}^{t-1} \log \frac{P_{\bm{\Theta}}(r_i \mid \vz_i)}{P_{\bm{\Theta}_\star}(r_i \mid \vz_i)}
        - H\,\bigl(P_{\bm{\Theta}}(r_t \mid \vz_t)\,\|\,P_{\bm{\Theta}_\star}(r_t \mid \vz_t)\bigr)^2\Biggr)
        \\
        &\le \cdots \le -\sum_{i=1}^{t} H\,\bigl(P_{\bm{\Theta}}(r_i \mid \vz_i)\,\|\,P_{\bm{\Theta}_\star}(r_i \mid \vz_i)\bigr)^2,
    \end{align}
    where $H(P\|Q)^2$ is the squared Hellinger distance between probability measures $P$ and $Q$ on $\Omega$, defined as
    \begin{equation}
    H(P\|Q)^2 := \int_{\Omega} \left(\sqrt{p(z)}-\sqrt{q(z)}\right)^2\, d\mu(z),
    \end{equation}
    with $p$ and $q$ denoting their respective densities with respect to a base measure $\mu$.
    
    We continue to lower-bound the Hellinger distance by
    \begin{align}
        \sum_{i=1}^{t} \Bigl(H\,\bigl(P_{\bm{\Theta}} (r_i \mid \vz_i)\,\|\,P_{\bm{\Theta}_\star}(r_i \mid \vz_i)\bigr)\Bigr)^2
        &\ge \sum_{i=1}^{t} \Bigl(\mathrm{TV}\,\bigl(P_{\bm{\Theta}}(r_i \mid \vz_i)\,\|\,P_{\bm{\Theta}_\star}(r_i \mid \vz_i)\bigr)\Bigr)^2
        \\
        &= \sum_{i=1}^{t} \Bigl(P_{\bm{\Theta}}(\vz_i) - P_{\bm{\Theta}_\star}(\vz_i)\Bigr)^2,
    \end{align}
    where the inequality uses the fact that for any distribution $p,q$, $H(p,q) \ge \mathrm{TV}(p,q)$~\citep[Theorem B.9]{Zhang_2023}.

    Then, by \cref{lem:martingale-exponential}, we obtain for each ${\bm{\Theta}} \in \Theta$, with probability at least $1-\delta$,
    \begin{align}
        \sum_{i=1}^{t} \log \frac{P_{\bm{\Theta}} (r_i \mid \vz_i)}{P_{\bm{\Theta}_\star}(r_i \mid \vz_i)}
        &\le \log\!\left(\frac{1}{\delta}\right)
        + \log \mathbb{E}\exp\!\left(\sum_{i=1}^{t} \log \frac{P_{\bm{\Theta}} (r_i \mid \vz_i)}{P_{\bm{\Theta}_\star}(r_i \mid \vz_i)}\right)
        \\
        &\le -\sum_{i=1}^{t} H\!\bigl(P_{\bm{\Theta}} (r_i \mid \vz_i)\,\|\,P_{\bm{\Theta}_\star}(r_i \mid \vz_i)\bigr)^2
        + \log\!\left(\frac{1}{\delta}\right)
        \\
        &\le -\sum_{i=1}^{t} \Bigl(P_{\bm{\Theta}} (\vz_i) - P_{\bm{\Theta}_\star}(\vz_i)\Bigr)^2
        + \log\!\left(\frac{1}{\delta}\right).
    \end{align}

\newpage
\section{Auxiliary Lemmas}

\begin{lemma}[Martingale Exponential Inequalities; Theorem 13.2 of \citet{Zhang_2023}]
\label{lem:martingale-exponential}
    Consider a sequence of random functions $\xi_1(\mathcal{Z}_1), \ldots, \xi_t(\mathcal{Z}_t), \ldots$ with respect to filtration $\{\mathcal{F}_t\}$. We have for any $\delta \in (0, 1)$ and $\lambda > 0$:
    \[
    \mathbb{P} \left( \exists n > 0 : - \sum_{i=1}^n \xi_i \ge \frac{\log(1/\delta)}{\lambda} + \frac{1}{\lambda} \sum_{i=1}^n \log \mathbb{E}_{Z_i^{(y)}} \exp(-\lambda \xi_i) \right) \le \delta,
    \]
    where $Z_t = (Z_t^{(x)}, Z_t^{(y)})$ and $\mathcal{Z}_t = (Z_1, \ldots, Z_t)$.
\end{lemma}

\begin{lemma}[Multiplicative Chernoff Bounds; Corollary 2.18 of \citet{Zhang_2023}]
\label{lem:multiplicative-chernoff}
    Assume that $X \in [0, 1]$ with $\mathbb{E}X = \mu$. Then for all $\epsilon > 0$,
    \begin{align}
        \mathbb{P}\left(\bar{X}_n \geq (1+\epsilon)\mu\right) &\leq \exp\left[\frac{-2n\mu\epsilon^2}{2+\epsilon}\right] \\
        \mathbb{P}\left(\bar{X}_n \leq (1-\epsilon)\mu\right) &\leq \exp\left[\frac{-2n\mu\epsilon^2}{2}\right].
    \end{align}
    Moreover, for $t > 0$, we have
    \[
        \mathbb{P}\left(\bar{X}_n \geq \mu + \sqrt{\frac{2\mu t}{n}} + \frac{t}{3n}\right) \leq \exp(-t).
    \]
\end{lemma}

\begin{lemma}[Elliptical Potential Lemma; Lemma 11 of \citet{abbasiyadkori2011linear}]
\label{lem:EPL}
   Let $\vx_1, \cdots, \vx_T \in \gB^d(X)$ be a sequence of vectors and $\mV_t := \lambda \mI + \sum_{s=1}^{t-1} \vx_s \vx_s^\intercal$.
    Then, we have 
    \begin{equation}
        \sum_{t=1}^T \min\left\{ 1, \lVert \vx_t \rVert^2_{\mV_t^{-1}} \right\} \leq 2 d \log\left( 1 + \frac{X^2 T}{d\lambda} \right).
    \end{equation}
\end{lemma}

\newpage
\section{Discussions on Regrets}
\label{app:other-regrets}
\subsection{Four Regret Definitions and Discussions}
A standard measure of performance in online learning is \emph{regret}. However, because the interaction is two-player and self-play, there are several ways to define regret, arising from different yet closely related communities: no-regret learning in games, dueling bandits/RL, and RL in two-player zero-sum games.
In the main text, we only consider $\MBRReg_{\eta}$, and so in this Appendix, we provide the deferred discussions regarding four regrets:

\begin{definition}
    The four regrets are defined as follows.
    \begin{enumerate}
        \item[(a)] \textbf{Average-Nash Regret}:
        \begin{equation}
            \ANReg_{\eta}(T) := \max_{\pi^1, \pi^2 \in \Pi} \sum_{t=1}^T \left\{ J_{\eta}(\pi^1, {\color{red}\hat{\pi}_t^2}) - J_{\eta}({\color{blue}\hat{\pi}_t^1}, \pi^2) \right\}.
        \end{equation}
        \item[(b)] \textbf{Average-Best-Response Regret}:
        \begin{equation}
            \ABRReg_{\eta}(T) := \sum_{t=1}^T \max_{\pi^1, \pi^2 \in \Pi} \left\{ J_{\eta}(\pi^1, {\color{red}\hat{\pi}_t^2}) - J_{\eta}({\color{blue}\hat{\pi}_t^1}, \pi^2) \right\}.
        \end{equation}
        \item[(c)] \textbf{Max-Nash Regret}:
        \begin{equation}
            \MNReg_{\eta}(T) := \max_{\pi \in \Pi} \sum_{t=1}^T \left\{ \frac{1}{2} - J_{\eta}({\color{blue}\hat{\pi}_t^1}, \pi) \right\}.
        \end{equation}
        \item[(d)] \textbf{Max-Best-Response Regret}:
        \begin{equation}
            \MBRReg_{\eta}(T) := \sum_{t=1}^T \max_{\pi \in \Pi} \left\{ \frac{1}{2} - J_{\eta}({\color{blue}\hat{\pi}_t^1}, \pi) \right\}.
        \end{equation}
    \end{enumerate}
    The unregularized variants are defined similarly and denoted without the $\eta$ subscript.
\end{definition}

\paragraph{Categorization Criteria.}
There are two criteria that determine each regret definition.
The first criterion is, at each time $t$, whether to consider the regrets of both players simultaneously, or to consider the regret of the max-player only.
This distinguishes between \emph{Average} or \emph{Max}.
The second criterion is whether to compare against a fixed comparator or to compare against the best response at each time $t$, which is usually time-varying.
This distinguishes between \emph{Nash} and \emph{Best-Response}.

Intuitively, the \textit{Average} regret definitions consider the ``suboptimality'' of both policies simultaneously. The difference between \textit{Nash} and \textit{Best-Response} is whether the regret is defined w.r.t. a fixed comparator (\textit{Nash}) or a dynamically changing comparator (\textit{Best-Response}).
Thus, in classical literature, they are also known as \emph{external} and \emph{internal (swap)} regrets.

\paragraph{Average Regrets.}
$\ANReg(T)$ is the notion originally considered in the seminal work of \citet{freund-schapire}, followed by numerous works on no-regret learning dynamics in games~\citep{daskalakis2011noregret,daskalakis2015noregret,daskalakis2018OGDA,rakhlin-sridharan,rakhlin-sridharan2,srygkanis2015games}, recently adopted to game-theoretic LLM alignment~\citep{zhang2025improving}.
$\ANReg(T)$ is precisely the regret considered in contextual dueling bandits~\citep[Eqn. (3)]{dudik2015dueling} and dueling RL~\citep[Eqn. (4)]{saha2023duelingrl}; indeed, for any $\pi \in \Pi$, utilizing the anti-symmetry of $J(\cdot, \cdot)$, we can rewrite
$J(\pi, {\color{red}\hat{\pi}_t^2}) - J({\color{blue}\hat{\pi}_t^1}, \pi) = J(\pi, {\color{blue}\hat{\pi}_t^1}) + J(\pi, {\color{red}\hat{\pi}_t^2}) - 1.$
This also slightly resembles Borda regret in dueling bandits~\citep{saha2021adversarial,wu2024dueling}, and average regret in dueling bandits under linear stochastic transitivity~\citep{saha2021dueling,bengs2021survey,bengs2022dueling}.

On the other hand, $\ABRReg(T)$ strongly resembles the notion of \textit{best response regret} in adversarial dueling bandits~\citep[Eqn. (1)]{saha2022contextual}, but there is a key difference.
In \citet{saha2022contextual}, the comparator at time $t$ is the same for both players, whereas in our regret setting, it differs for each player.
Basically, each player must compete with the worst-case (strongest) adversary from her perspective, who chooses the best response from his perspective.

\paragraph{Max Regrets.}
The notion of considering the regret of the \textit{max} player dates back to the self-play framework for RL in two-player zero-sum games~\citep{bai2020self-play,bai2020self-play2,liu2021self-play,jin2022exploiter,xiong2022self-play}; this idea has been recently applied to theoretical analyses of online RLHF under general preference~\citep{ye2024general,wu2025greedy}.
Basically, the intuition is that the learner only cares about obtaining the NE policy for the max-player, which is the policy that is actually deployed in practice.

Note that the min-player's policies ${\color{red}\hat{\pi}_t^2}$ do not contribute to the regret at all, and thus, often, the min-player acts as an exploration agent whose sole role is to collect as much information as possible to facilitate the learning of the max-player~\citep{bai2020self-play,bai2020self-play2,liu2021self-play,jin2022exploiter,xiong2022self-play,xiong2024iterative,ye2024general}.

\subsection{Online-to-Batch Conversion}
\label{app:online-to-batch}
A standard consequence of no-regret learning in repeated zero-sum games is an \emph{online-to-batch conversion}: the time-averaged (mixed) policies form an approximate Nash equilibrium. We formalize this statement in the following proposition.

\begin{proposition}[Online-to-batch conversion]
\label{prop:online-to-batch}
Let $\{(\hat{\pi}_t^1,\hat{\pi}_t^2)\}_{t=1}^T\subseteq \Pi\times\Pi$ be any policy sequence.
Define the uniform mixture policies as $\bar{\pi}_T^i := \frac1T\sum_{t=1}^T \hat{\pi}_t^i$ for $i \in \{1, 2\}$ and for simplicity, let us denote $\bar{\pi}_T := \bar{\pi}_T^1$.

\smallskip
\noindent\textbf{(a) Average regrets.}
For average regrets, $(\bar{\pi}_T^1, \bar{\pi}_T^2)$ is a $\frac{\Reg(T)}{T}$-approximate symmetric NE:
\[
\max_{\pi^1,\pi^2\in\Pi}
\Bigl\{J_{\eta}(\pi^1,\bar{\pi}_T^2)-J_{\eta}(\bar{\pi}_T^1,\pi^2)\Bigr\}
\;\le\; \frac{\ANReg_{\eta}(T)}{T}
\;\le\; \frac{\ABRReg_{\eta}(T)}{T}.
\]

\smallskip
\noindent\textbf{(b) Max regrets.}
For max regrets, $\bar{\pi}_T$ is a $\frac{2 \Reg(T)}{T}$-approximate symmetric NE:
\[
\max_{\pi\in\Pi}\Bigl\{J_{\eta}(\pi,\bar{\pi}_T)-J_{\eta}(\bar{\pi}_T,\pi)\Bigr\}
\;\le\; \frac{2\MNReg_{\eta}(T)}{T}
\;\le\; \frac{2\MBRReg_{\eta}(T)}{T}.
\]
\end{proposition}

\begin{proof}
\textbf{(a)} Fix any $\pi^1,\pi^2\in\Pi$.
By the bilinearity of $J$ and Jensen's inequality w.r.t. $\psi(\cdot)$,
\[
    J_{\eta}(\pi^1,\bar{\pi}_T^2)
    = J(\pi^1,\bar{\pi}_T^2) - \eta^{-1}\psi(\pi^1) + \eta^{-1}\psi(\bar{\pi}_T^2)
    \le \frac1T\sum_{t=1}^T J_{\eta}(\pi^1,\hat{\pi}_t^2),
\]
and similarly,
\[
    J_{\eta}(\bar{\pi}_T^1,\pi^2)
    = J(\bar{\pi}_T^1,\pi^2) - \eta^{-1}\psi(\bar{\pi}_T^1) + \eta^{-1}\psi(\pi^2)
    \ge \frac1T\sum_{t=1}^T J_{\eta}(\hat{\pi}_t^1,\pi^2).
\]
Subtracting the two inequalities and taking $\max_{\pi^1,\pi^2}$ yields
\begin{align}
    \max_{\pi^1,\pi^2}\Bigl\{J_{\eta}(\pi^1,\bar{\pi}_T^2)-J_{\eta}(\bar{\pi}_T^1,\pi^2)\Bigr\}
    & \le \frac{1}{T} \max_{\pi^1,\pi^2}\sum_{t=1}^T\Bigl(J_{\eta}(\pi^1,\hat{\pi}_t^2)-J_{\eta}(\hat{\pi}_t^1,\pi^2)\Bigr) \\
    & \leq \frac{1}{T}\sum_{t=1}^T\max_{\pi^1,\pi^2}\Bigl(J_{\eta}(\pi^1,\hat{\pi}_t^2)-J_{\eta}(\hat{\pi}_t^1,\pi^2)\Bigr)
\end{align}

\textbf{(b)} By the same averaging argument (bilinearity of $J$ and Jensen's inequality),
for every $\pi\in\Pi$,
\[
J_{\eta}(\pi,\bar{\pi}_T)-J_{\eta}(\bar{\pi}_T,\pi)
\le \frac{1}{T}\sum_{t=1}^T \Bigl(J_{\eta}(\pi,\hat{\pi}_t^1)-J_{\eta}(\hat{\pi}_t^1,\pi)\Bigr).
\]
Since for each $t$, $J_{\eta}(\pi,\hat{\pi}_t^1)-J_{\eta}(\hat{\pi}_t^1,\pi)
= 2\Bigl(\tfrac12 - J_{\eta}(\hat{\pi}_t^1,\pi)\Bigr)$,
taking $\max_\pi$ and substituting yields
\begin{align}
    \max_{\pi}\bigl\{J_{\eta}(\pi,\bar{\pi}_T)-J_{\eta}(\bar{\pi}_T,\pi)\bigr\} \le \frac{2}{T}\max_{\pi}\sum_{t=1}^T\Bigl(\tfrac12 - J_{\eta}(\hat{\pi}_t^1,\pi)\Bigr)  \leq \frac{2}{T}\sum_{t=1}^T\max_{\pi}\Bigl(\tfrac12 - J_{\eta}(\hat{\pi}_t^1,\pi)\Bigr). \tag*{\qedhere}
\end{align}
\end{proof}

\newpage
\section{Synthetic Experiments}
\label{app:experiment}

In this section, we present empirical results to numerically validate the theoretical regret bounds of $\texttt{Greedy Sampling (GS)}$ established in \cref{sec:regret-logarithmic}.

\subsection{Experiment Setup.}
\label{app:experiment-setup}
We consider a bilinear preference model with logistic link function $\mu(z) = 1 / (1 + e^{-z})$ and reverse KL regularizer. We use $K$ feature vectors in an uncontextualized setting, randomly sampled via a uniform distribution, and normalized to $\ell_2$
norm $\leq1$. 

We use the following hyperparameters for our experiments:
\begin{itemize}
    \item $d=5, \, K=20, \, S=5$
    \item $d=10, \, K=40, \, S=10$
    \item $r=1$
    \item $\eta \in \{10^{-2}, \cdots, 10^4 \}$
    \item $T = 10000$
\end{itemize}

All reported metrics are averaged over 20 independent random seeds, with standard deviations shown as shaded regions or error bars.

\subsection{Implementation Details and Reproducibility}
\label{app:implementation-details}

To ensure reproducibility and bridge the gap between continuous bounds and discrete floating-point arithmetic, we detail our experimental setup and numerical stabilizations (source code provided in the supplement). 
\paragraph{Instance Generation.} To ensure that the true preference matrix $\bm{\Theta}_\star$ rigorously satisfies the theoretical assumptions of being low-rank, exactly skew-symmetric ($\bm{\Theta}_\star + \bm{\Theta}_\star^\top = 0$), and bounded in norm, we construct it using its real spectral decomposition. We first draw a Gaussian matrix $\mathbf{G} \in \mathbb{R}^{d \times 2r}$ and compute its QR decomposition to obtain an orthonormal basis matrix $\mathbf{Q} \in \mathbb{R}^{d \times 2r}$. We then construct a block-diagonal skew-symmetric core matrix $\mathbf{B} \in \mathbb{R}^{2r \times 2r}$ consisting of $r$ independent $2 \times 2$ blocks of the form $\left[\begin{smallmatrix} 0 & s_i \\ -s_i & 0 \end{smallmatrix}\right]$, where the singular values $s_i$ are sampled uniformly from $[0.1, 1.0]$. The unnormalized parameter matrix is assembled as $\mathbf{\Theta} = \mathbf{Q} \mathbf{B} \mathbf{Q}^\top$, guaranteeing an exact rank of $2r$. We apply a minor anti-symmetrization $\frac{1}{2}(\mathbf{\Theta} - \mathbf{\Theta}^\top)$ solely to correct floating-point inaccuracies, and scale the matrix such that its Frobenius norm is exactly $\|\bm{\Theta}_\star\|_F = S$. 

\paragraph{Equiliibrium and Estimation Procedures.} Both the base exploration policy $\pi_0$ and the reference policy $\pi_{\text{ref}}$ are initialized as the uniform distribution over the $K$ available items ($\pi_0 = \pi_{\text{ref}} = \mathbf{1}/K$). Since our experiments focus on the reverse-KL regularizer, the regularized game admits exact log-ratio coordinates, which allows us to compute equilibria utilizing SciPy's \texttt{root} function with the \texttt{hybr} method. 
Computing fixed points and log-likelihoods in extreme regimes ($\eta \gg 1$) requires specific numerical heuristics. Directly solving the fixed-point equation $p = \text{BR}_\eta(p)$—where $\text{BR}_\eta$ denotes the regularized best response operator against an opponent policy—is highly unstable in this low-temperature limit due to exponential sensitivity. To address this, we employ an $\eta$-continuation (homotopy) method, sequentially solving the fixed point via adaptively damped Mann iterations. 

\paragraph{Numerical Stability and Reproducibility.} We use several numerical safeguards to preserve stability. For preference estimation, we utilize the Online Newton Step (ONS) algorithm to sequentially update the estimated parameter matrix, alongside offline Maximum Likelihood Estimation (MLE). During these updates, inner products are clipped to $[-50, 50]$ to prevent exponential overflow in the link function, and the ONS logistic variance term (which acts as the Hessian approximation) is strictly bounded to $[10^{-6}, 0.25]$ to prevent inverse covariance degeneracy. 

All experiments were executed on standard consumer-grade CPU hardware without the need for hardware accelerators. The complete evaluation suite, including the 20 independent random seeds across all dimension and regularization configurations, executes to completion within a few hours. 
Full code can be found in \url{https://github.com/minju-hong/online_rlhf_gbpm.git}.

\subsection{Main Results}
\label{app:exp-results}

We first plot the cumulative $\MBRReg$ for varying regularization strengths $\eta$. As illustrated in \cref{fig:1}, the regret tightly fits a $\log T$ curve for small $\eta$ and shifts to a $\sqrt{T}$ curve for large $\eta$, corroborating the bounds established in \cref{thm:regularized-log}.

To quantify this phase transition, we fit the empirical regret to both logarithmic and square-root models and plot the goodness-of-fit ($R^2$)
in \cref{fig:2}. Equating the two upper bounds in
\cref{thm:regularized-log} gives the theoretical crossover
\[
    \eta_{\mathrm{cross}}(d,T)
    =
    \frac{\kappa^{1/2} C_{\min}^{1/2}}{\beta}
    \frac{\sqrt{T}}{d^2 \log(T/d)}.
\]
Thus, for $\eta \gtrsim \eta_{\mathrm{cross}}(d,T)$, the
$\sqrt{T}$ term is selected by the minimum, whereas for $\eta \lesssim \eta_{\mathrm{cross}}(d,T)$, the logarithmic term is selected. Ignoring constants and the mild logarithmic dependence, this threshold scales as $d^{-2}\sqrt{T}$ when $C_{\min}$ is held fixed. Accordingly, increasing the dimension from $d=5$ to $d=10$ shifts the predicted crossover toward smaller $\eta$, up to the dependence of $C_{\min}$ on $d$.

Finally, we show the cumulative regret at $T=10^4$ scaling with respect to $\eta$ (\cref{fig:3}). The regret initially grows linearly but strictly plateaus in the large-$\eta$ regime, perfectly matching the behavior of our unified $\min(\cdot, \cdot)$ bound.

\begin{figure}[h]
    \centering
    \begin{minipage}[b]{0.45\textwidth}
        \centering
        \includegraphics[width=\textwidth]{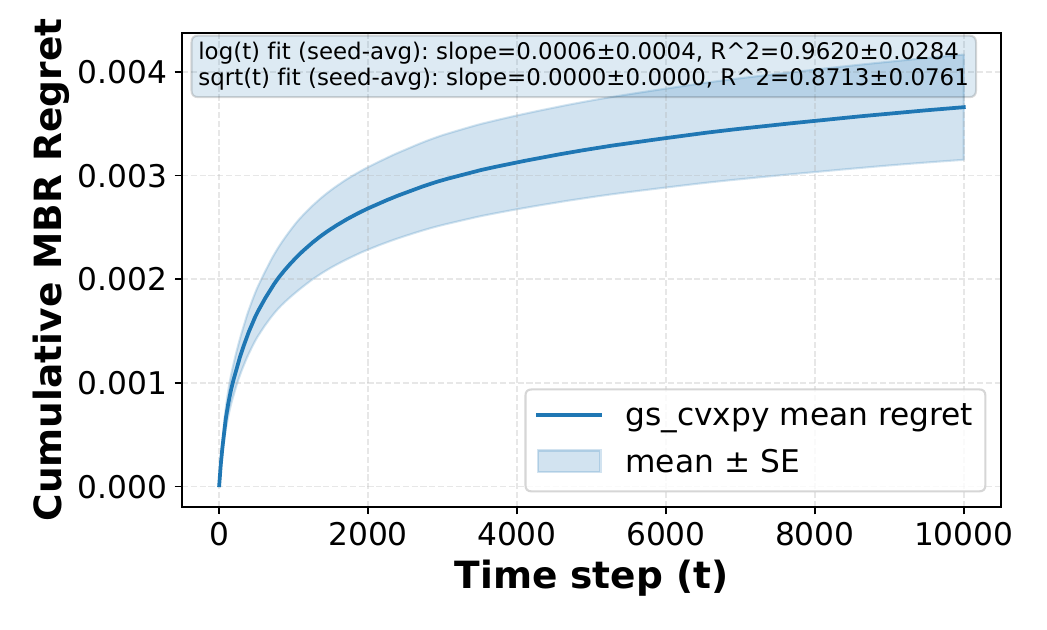}
        \vspace{0.1cm}
        
        \small (a) Small $\eta=0.01$: $\log T$ regime
        \label{fig:1a}
    \end{minipage}
    \hfill
    \begin{minipage}[b]{0.45\textwidth}
        \centering
        \includegraphics[width=\textwidth]{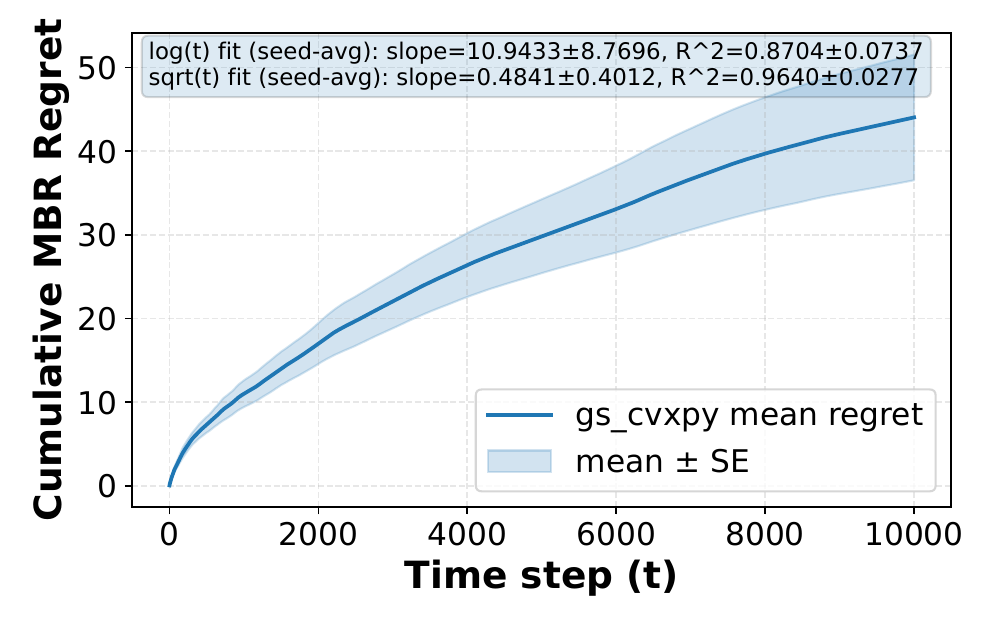}
        \vspace{0.1cm}
        
        \small (b) Large $\eta=1000$: $\sqrt{T}$ regime
        \label{fig:1b}
    \end{minipage}
    \vspace{-0.15cm}
    \caption{MBR Regret Trajectories.} 
    \label{fig:1}
\end{figure}

\begin{figure}[h]
    \centering
    \begin{minipage}[b]{0.45\textwidth}
        \centering
        \includegraphics[width=\textwidth]{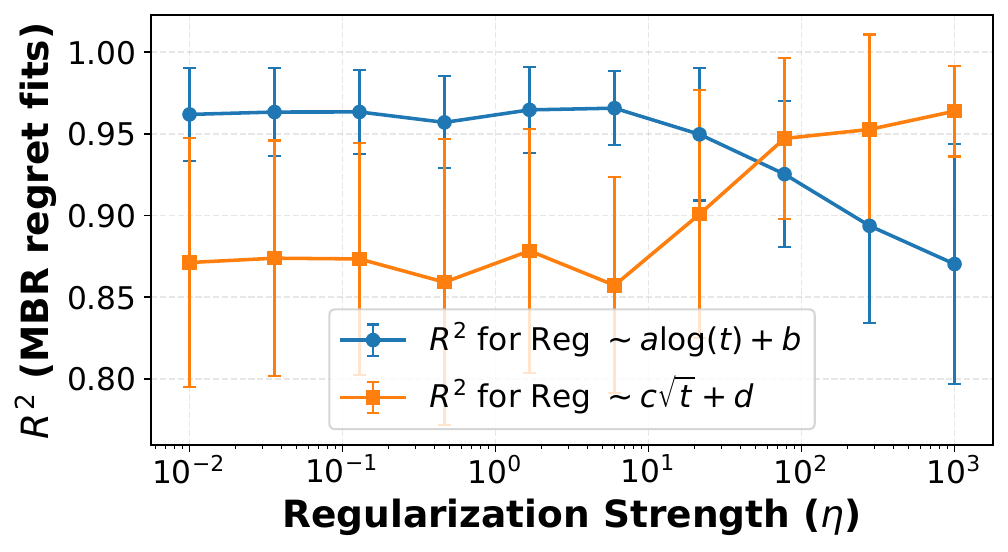}
        \vspace{0.1cm}
        
        \small (a) $d=5$
        \label{fig:2a}
    \end{minipage}
    \hfill
    \begin{minipage}[b]{0.45\textwidth}
        \centering
        \includegraphics[width=\textwidth]{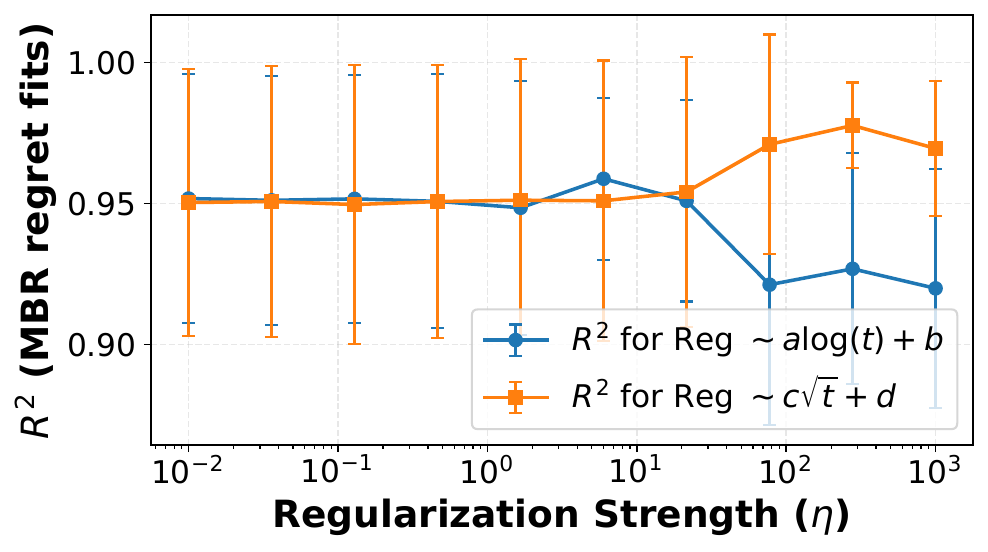}
        \vspace{0.1cm}
        
        \small (b) $d=10$
        \label{fig:2b}
    \end{minipage}
    \vspace{-0.15cm}
    \caption{Crossover Point Analysis.} 
    \label{fig:2}
\end{figure}

\begin{figure}[h!]
    \centering
    \includegraphics[width=0.8\textwidth]{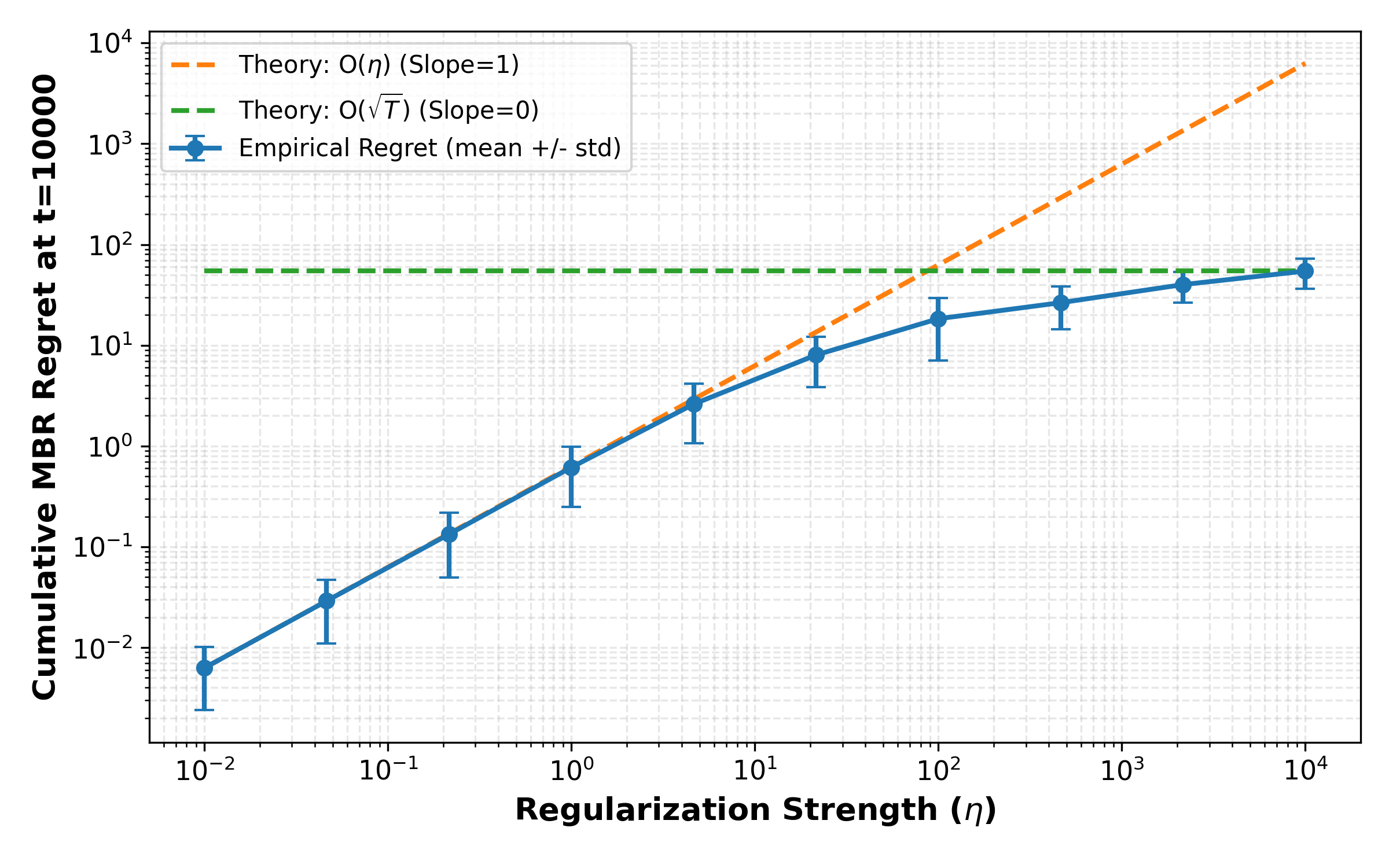}
    \vspace{-0.15cm}
    \caption{Final Regret Scaling.}
    \label{fig:3}
\end{figure}
\clearpage

\newpage
\section{Future Directions}
\label{app:future}

\paragraph{Relaxing the Feature Diversity Assumption.}
Our current regret bounds rely on the feature diversity assumption (\cref{ass:REC}), characterized by the minimum eigenvalue $C_{\min}$.
While this assumption is standard in the contextual bandits literature that involves either greedy sampling (e.g., algorithms without sophisticated exploration strategies) or high dimensions, it may still be restrictive for general RLHF applications.
Recent works have investigated minimal assumptions required for greedy strategies, such as the local anti-concentration (LAC) property proposed by \citet{kim2024greedy}.
Investigating the impact of such relaxed conditions on online RLHF with \GBPM~ (e.g., whether we can still obtain polylogarithmic regret with \texttt{GS}) remains an important open question.

\paragraph{Instance-Specific Guarantees for Unregularized Regret.}
While our work establishes $\tilde{\gO}(\sqrt{T})$ guarantees for \textit{unregularized} regret (via \cref{thm:regularized-log}), these bounds reflect worst-case hardness.
For instance, \citet{ito2025twoplayer} demonstrated that in tabular games with bandit feedback, the Nash regret for the Tsallis-INF algorithm~\citep{tsallis1988,abernethy2015tsallis,zimmert-seldin} scales with the ``sparsity'' or ``entropy'' of the NE set, potentially achieving logarithmic regret $\gO(\log T)$ when the NE is unique and deterministic (a pure strategy), and even rates of the form $\gO(T^c)$ for some $c \in (0, 1)$, depending on the geometry of the set of Nash Equilibria.
Adapting such instance-dependent guarantees to the contextual~\GBPM~setting is non-trivial, even when the link function $\mu$ is linear.
Recent advances in ``Best-of-Both-Worlds'' algorithms for linear contextual bandits \citep{kuroki2024bow,kato2025tsallis-inf} may provide a promising starting point.

\paragraph{Computationally Efficient Algorithms.}
Our current theoretical framework assumes access to a computational oracle for finding the NE (\cref{oracle:NE}), which may be computationally expensive in practice. 
Developing efficient variants of our algorithms is a practical priority.
Promising approaches include leveraging online estimation techniques such as Online Mirror Descent (OMD) \citep{zhang2025onepass}, minimax optimization techniques such as optimistic OMD \citep{rakhlin-sridharan,srygkanis2015games,zhang2025improving}, or reductions to offline/online regression oracles \citep{foster2020beyond}.

\end{document}